\documentclass{article} 
\usepackage{iclr2026_conference,times}
\usepackage{amsmath, amssymb, amsfonts} 
\usepackage{etoolbox}

\usepackage{amsthm}
\newtheorem{definition}{Definition} 
\usepackage[dvipsnames]{xcolor}
\iclrfinalcopy
\makeatletter
\patchcmd{\@maketitle}
  {Published as a conference paper at ICLR 2026}
  {Paper under submission}
  {}{}
\makeatother
\usepackage[utf8]{inputenc} 
\usepackage{graphicx}
\usepackage{booktabs}     
\usepackage{enumitem}
\newtheorem{theorem}{Theorem}
\usepackage{algorithm}
\usepackage{adjustbox}
\usepackage{underscore}
\usepackage{caption}  
\usepackage{algpseudocode}  

\usepackage{amsmath,amsfonts,bm}

\def\eqref#1{equation~\ref{#1}}

\def\1{\bm{1}}

\DeclareMathAlphabet{\mathsfit}{\encodingdefault}{\sfdefault}{m}{sl}
\SetMathAlphabet{\mathsfit}{bold}{\encodingdefault}{\sfdefault}{bx}{n}

\newcommand{\Var}{\mathrm{Var}}

\DeclareMathOperator*{\argmax}{arg\,max}
\DeclareMathOperator*{\argmin}{arg\,min}

\newcommand{\ProposeAndVerify}

\usepackage{hyperref}
\usepackage{listings}

\usepackage{pgfplots}
\pgfplotsset{compat=1.18}
\usetikzlibrary{arrows.meta}
\usepgfplotslibrary{groupplots} 
\usepackage{url}
\usepackage{multirow}
\usepackage{array}
\usepackage[a4paper, margin=1in]{geometry}
\definecolor{Firebrick}{RGB}{178,34,34}

\usepackage[svgnames]{xcolor} 
\usepackage[most]{tcolorbox}  
\usepackage{amsmath}          
\usepackage{graphicx}         
\usepackage{adjustbox}        

\newtcolorbox{questionbox}[1][]{
  enhanced,
  breakable,
  colback=LightBlue!10, 
  colframe=SteelBlue,   
  fonttitle=\bfseries,
  title=Question,
  attach boxed title to top left={yshift=-2mm, xshift=3mm},
  boxed title style={colback=SteelBlue, sharp corners},
  #1
}

\newtcolorbox{correctbox}[1][]{
  enhanced,
  breakable,
  colback=LimeGreen!10, 
  colframe=DarkGreen,   
  fonttitle=\bfseries,
  title=Our Method (Correct),
  attach boxed title to top left={yshift=-2mm, xshift=3mm},
  boxed title style={colback=DarkGreen, sharp corners},
  #1
}

\newtcolorbox{incorrectbox}[1][]{
  enhanced,
  breakable,
  colback=Tomato!10,   
  colframe=Firebrick, 
  fonttitle=\bfseries,
  title=Incorrect Approaches (AFlow \& MaAS),
  attach boxed title to top left={yshift=-2mm, xshift=3mm},
  boxed title style={colback=Firebrick, sharp corners},
  #1
}

\newcommand{\finalanswer}[1]{\tcbox[sharp corners, colback=yellow!30, colframe=black!75, boxrule=0.5pt, nobeforeafter]{#1}}

\usepackage[table]{xcolor}
\usepackage{pgfplots}
\usepackage{wrapfig}   
\pgfplotsset{compat=1.18} 
\usetikzlibrary{arrows.meta}
\definecolor{myblue}{RGB}{55,126,184}
\definecolor{myred}{RGB}{228,26,28}
\definecolor{myorange}{RGB}{255,127,0}

\definecolor{promptbg}{RGB}{248,248,248}
\definecolor{promptframe}{gray}{0.8}
\definecolor{prompttext}{RGB}{40,40,40}
\definecolor{promptkeyword}{RGB}{153,0,0}
\usepackage{listingsutf8} 
\lstdefinestyle{promptstyle}{
    backgroundcolor=\color{promptbg},
    basicstyle=\ttfamily\footnotesize\color{prompttext},
    frame=single,
    rulecolor=\color{promptframe},
    frameround=tttt,
    breaklines=true,            
    breakatwhitespace=true,     
    columns=fullflexible,
    showstringspaces=false,
    keywordstyle=\color{promptkeyword}\bfseries,
    commentstyle=\color{gray}\itshape,
    stringstyle=\color{blue!50!black},
    tabsize=2,
    xleftmargin=0.5em,          
    xrightmargin=0.5em,         
    framexleftmargin=0.5em,     
    framexrightmargin=0.5em
}
\usepackage{tikz}
\usetikzlibrary{
    positioning,          
    shapes.geometric,     
    decorations.pathreplacing, 
    fit,                  
    backgrounds           
}

\definecolor{tableheader}{gray}{0.92}
\definecolor{runnerupgray}{gray}{0.95} 
\definecolor{sotagray}{gray}{0.88}     

\definecolor{primaryColor}{RGB}{28, 128, 207} 

\tikzstyle{agent-active} = [rectangle, rounded corners, draw=black, fill=blue!20, thick, minimum size=12mm, text width=2.5cm, align=center]
\tikzstyle{agent-inactive} = [rectangle, rounded corners, draw=gray, fill=gray!10, thick, minimum size=12mm, text width=2.5cm, align=center, dashed]
\tikzstyle{feedback-text} = [font=\small, text=red!80!black]
\tikzstyle{arrow-forward} = [->, thick, color=blue!80]
\tikzstyle{arrow-feedback} = [->, thick, color=red!80]
\tikzstyle{arrow-neutral} = [->, thick, color=gray, dashed]

\usepackage{xcolor} 
\usepackage{soul}   

\title{Failure-Driven Workflow Refinement}

\author{
    Jusheng Zhang\textsuperscript{1},
    Kaitong Cai\textsuperscript{1},
    Qinglin Zeng\textsuperscript{1},
    Ningyuan Liu\textsuperscript{1},
    Yijia Fan\textsuperscript{1},
    Ziliang Chen\textsuperscript{1}  
    Keze Wang\textsuperscript{1,2}\thanks{Corresponding author: \texttt{kezewang@gmail.com}} \\
    \textsuperscript{1}Sun Yat-sen University 
    \textsuperscript{2}X-Era AI Lab
}

\begin{document}

\maketitle
\begin{abstract}
Optimizing LLM-based workflows is typically formulated as a global search, where candidate workflows are evaluated based on a scalar metric. This paradigm, however, suffers from a critical flaw: information collapse. By reducing rich, multi-step execution traces to simple success/failure signals, existing methods are rendered blind to the underlying structure of failures, fundamentally preventing them from modeling the workflow's failure distribution.
We reconceptualize this challenge as a distributional problem. We propose a new paradigm where the optimization goal is not to maximize a scalar score, but to directly minimize a workflow's \textbf{Expected Failure Mass}, i.e., the integral of its failure probability density function defined over a high-dimensional \textbf{Failure Signature Space} ($\mathcal{F}$). This distributional lens allows us to move from inefficient, zero-order optimization to a principled, gradient-like descent on the failure landscape itself.
We introduce \textbf{CE-Graph}, a framework that operationalizes this paradigm through a novel, failure-driven refinement process. CE-Graph approximates the failure distribution from a pool of counterexamples, identifies its densest regions as recurring failure modes, and applies targeted, operator-constrained graph edits via a \textbf{Propose-and-Verify} mechanism to greedily reduce the failure mass. On math, code, and QA benchmarks, our CE-Graph achieves higher robustness at a significantly lower cost than strong baselines. This suggests that a system's reliability emerges not from avoiding failures, but from systematically learning and reshaping the geometric structure of its failure distributions.
\end{abstract}

\section{Introduction}
\label{sec:introduction}

Large Language Models (LLMs)~\citep{GPT4,zys,LLMZS1,LLMZS2} are increasingly used to power agentic workflows, where tasks are decomposed into multiple steps involving tool calls, logical control flows, and verification routines~\citep{zhang2025kabb,MAAS,AFlow,jeyakumar2024advancing,Z5}. Such workflows enable LLMs to handle long-horizon reasoning and complex problem-solving. The central challenge is how to optimize these workflows: given a dataset $D$, how can we construct a workflow $W$ that achieves the highest reliability?

This challenge is fundamentally distinct from traditional program synthesis~\citep{Yu_2025,Sapkota_2026}. Unlike programs in a formal language, LLM-based workflows operate in a vast, unstructured space defined by natural language prompts and arbitrary tool interactions. Furthermore, their failures are not deterministic bugs but are often \textbf{stochastic and semantic} in nature, stemming from subtle flaws in the model's reasoning process~\citep{steenhoek2025errmachinevulnerabilitydetection,song2025a,AAABBBB,Z2,Z3}. Consequently, established methods designed for formal program spaces are often ineffective, necessitating a new optimization paradigm.

A common formulation treats workflow optimization as a global program search~\citep{AFlow,MAAS,li2024autoflowautomatedworkflowgeneration,11082076}. The objective is to find a configuration $W^*$ that maximizes a global, sample-based evaluation metric $G(W,D)$:
$\label{eq:intro_global_search}
W^* = \underset{W \in S}{\arg\max} \; G(W,D).
$
Recent systems exemplify this view by employing advanced search strategies such as Monte Carlo Tree Search (MCTS)~\citep{MCTS1,MCTS2,MCTS3}. While effective, this line of work inherits a core limitation rooted in its objective function, i.e., the reliance on \textbf{information collapse}. The rich, multi-step trace of each failure is compressed into a single binary signal, preventing the optimizer from modeling the \textbf{underlying failure distribution}, which is key to efficient refinement.

We argue that these failures are not random noise but are, in fact, samples from a workflow-specific failure distribution~\citep{pan2025why,zhang2025which}. This insight motivates a paradigm shift. We re-conceptualize workflow optimization not as maximizing a scalar score, but as a process of progressively reshaping this underlying failure distribution. We formalize this by introducing the \textbf{Failure Signature Space} $\mathcal{F}$ and casting the optimization goal as the minimization of the \textbf{Expected Failure Mass} $M(W)$:
$\label{eq:intro_failure_mass}
W^* = \underset{W \in S}{\arg\min} \; M(W) \quad \text{where} \quad M(W) = \int_{\mathcal{F}} p(\mathbf{s} \mid W) \,d\mathbf{s}.
$This new formulation, which aims to minimize the integral of a failure density function, marks a fundamental departure from the conventional paradigm of maximizing a sample-based performance metric. Instead of treating the workflow as a black box, it provides a white-box and distributional lens. This allows us to move from inefficient, zero-order optimization to a principled, gradient-like descent directly on the landscape of failure modes.
To operationalize this principle, we propose CE-Graph (Counterexample-Guided Workflow Optimization), a framework that realizes this failure-mass-reduction strategy. It operates through three key stages. First, it maintains a counterexample pool and systematically diagnoses the densest regions of the failure distribution by clustering their \textbf{semantic signatures}, thereby approximating the "semantic gradient". Second, it replaces heuristic fixes with a principled \textbf{Propose-and-Verify} mechanism, which generates and empirically validates targeted graph edits designed to reduce the mass of these failure modes. Third, our CE-Graph incorporates a convergence-aware stopping rule to halt optimization once the failure distribution has stabilized.

We evaluate our CE-Graph on math, code, and QA benchmarks, showing that it achieves higher success rates at significantly lower optimization cost compared to strong workflow search baselines. Beyond these empirical gains, our work highlights a fundamental principle for building robust agentic systems: true reliability emerges not from attempting to avoid failures, but from systematically understanding and resolving their underlying distributional structure. Conceptually, CE-Graph establishes a \textbf{self-referential optimization paradigm}, employing language models to analyze and refine the failure distributions of the very LLM-based systems they compose.
\section{A New Optimization Paradigm for LLM Workflows}
\label{sec:paradigm}
In this section, we develop the theoretical foundation for our CE-Graph by proposing a new paradigm for LLM workflow optimization. We begin by introducing the concept of a \textbf{Failure Signature Space} and define the optimization problem as minimizing the \textit{expected failure mass} within this space~\citep{ribeiro2020beyond}. This distributional perspective addresses the limitations of traditional scalar-based metrics by preserving the rich structure of failures. We then formalize the dominant ``global search'' paradigm and demonstrate its mathematical limitations in solving this objective, particularly its blindness to failure distributions~\citep{AFlow}. This motivates our new paradigm, ``failure-driven refinement'', which we articulate as a principled, iterative approach to minimizing this failure mass, thereby setting a rigorous stage for the algorithmic framework in Section~\ref{sec:method}. Our analysis reveals that true robustness in LLM workflows emerges from reshaping failure landscapes, rather than merely aggregating success rates.\begin{wrapfigure}{r}{0.33\textwidth} 
  \centering
  \vspace{-15pt}
  \includegraphics[width=\linewidth]{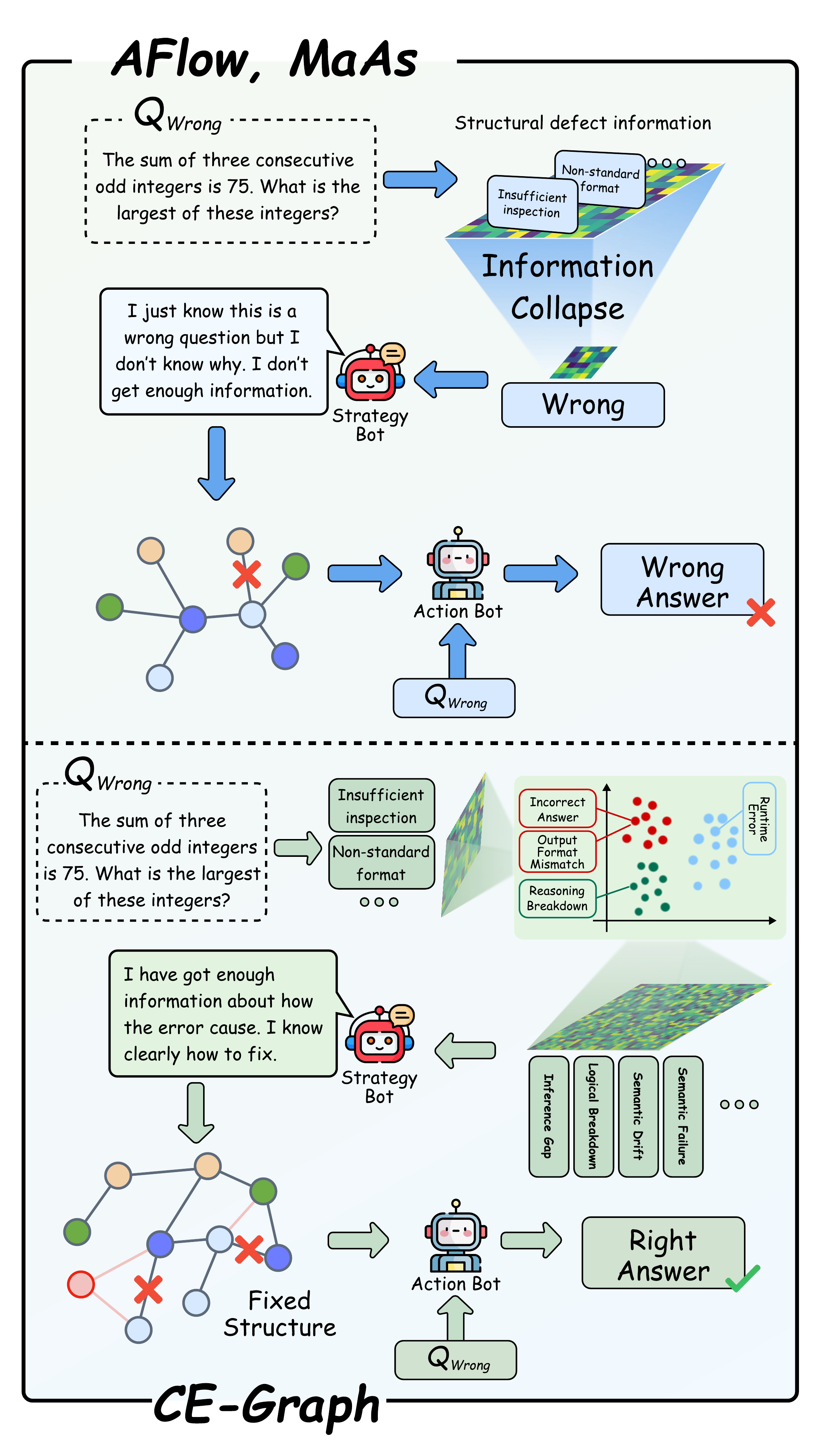} 
\caption{Upper: traditional LLMs compress rich error traces into binary signals, causing ``information collapse'' and obscuring the underlying failure distribution, which hinders systematic refinement. Lower: \textbf{CE-Graph} leverages the Failure Signature Space to cluster errors into coherent modes}
  \label{fig:contrast}
\end{wrapfigure}
  \vspace{-20pt}

\subsection{Problem Formulation in the Failure Signature Space}
\label{sec:problem_formulation}
The raw execution traces of LLM workflows are complex, high-dimensional, and unstructured, making direct comparison and analysis difficult. For instance, in a math reasoning task like GSM8K~\citep{GSM8K}, a failure trace might involve multiple steps of erroneous calculations, tool calls, and reasoning flaws, which are hard to quantify without a unified representation. To enable systematic optimization, we map failures into a shared, structured vector space. Let $\mathcal{F}$ be this high-dimensional \textbf{Failure Signature Space}, where each point $\mathbf{s} \in \mathcal{F}$ represents a unique, structured signature of a potential workflow failure. Any given workflow $W$ induces a probability density function $p(\mathbf{s} \mid W)$ over this space. This density function serves as a topographical map of the workflow's weaknesses: for a perfect workflow, the map is flat with $p(\mathbf{s} \mid W) = 0$; for a flawed workflow, ``mountains'' of high density emerge in regions corresponding to its systematic errors, such as recurring arithmetic mistakes in math problems.
By modeling failures distributionally, we capture not just the frequency but the geometric clustering of errors, enabling targeted repairs. The ultimate goal of optimization is to find a workflow $W^*$ that flattens this landscape, minimizing the total probability mass of failure. We formalize this as minimizing the \textbf{Expected Failure Mass} $M(W)$:
\begin{equation}
W^* = \argmin_{W \in S} M(W) \quad \text{where} \quad M(W) = \int_{\mathcal{F}} p(\mathbf{s} \mid W) \, d\mathbf{s}.
\label{eq:failure_mass_objective}
\end{equation}

This formulation provides a principled theoretical lens for analyzing different optimization paradigms, shifting the focus from black-box evaluation to white-box distributional refinement.
This objective relies on several assumptions: (i) failures are independent and identically distributed (i.i.d.) samples from $p(\mathbf{s} \mid W)$, (ii) $\mathcal{F}$ captures all relevant semantic and structural aspects of failures via an appropriate embedding, and (iii) the density $p(\mathbf{s} \mid W)$ is estimable from finite counterexamples. Limitations include the curse of dimensionality in $\mathcal{F}$ and the non-convexity of $M(W)$, which may cause local minima. Nevertheless, this lens still enables gradient-like optimization, as explored next.

\subsection{The Pitfalls of Global Search}
\label{sec:limitations_search}

The dominant global search paradigm, exemplified by methods like Monte Carlo Tree Search (MCTS) in systems such as MaAS~\citep{MAAS} or AFlow~\citep{AFlow}, attempts to minimize failure mass indirectly. It relies on a global performance metric $G(W, D)$, which approximates $1 - M(W)$ via black-box~\citep{hhcy,hhcy2} sampling, i.e., essentially averaging success rates over a dataset $D$. The objective is thus to maximize this score: $W^* = \argmax_{W \in S} G(W, D)$.

Although intuitive, this approach suffers from a critical flaw: information collapse. The optimizer accesses only a collapsed binary signal for each execution trace $\tau_d$, i.e., $\text{Observe}(\tau_d) \in \{0, 1\}$, discarding the rich details of why and where the failure occurred. For example, in code generation tasks like HumanEval~\citep{HumanEval}, multiple failures might stem from similar semantic issues (e.g., off-by-one errors), but global search treats them as independent binary outcomes, preventing the modeling of underlying patterns. It lacks details about the failure's location $\mathbf{s}$ in $\mathcal{F}$, preventing estimation of $p(\mathbf{s} \mid W)$'s shape. Essentially, the optimizer is blindfolded, unable to detect ``mountains'' on the failure map and forced to wander randomly. This reduces $M(W)$ to an opaque scalar, resulting in inefficient zero-order optimization in a vast search space.

Quantitatively, this inefficiency manifests as high sample complexity: global methods often require thousands of evaluations to converge, as they cannot exploit failure structure~\citep{snoek2012practicalbayesianoptimizationmachine}. In contrast, our paradigm leverages distributional insights to achieve faster, more targeted improvements, as detailed in Section~\ref{sec:repair_paradigm}.

\subsection{The Failure-Driven Refinement Paradigm}
\label{sec:repair_paradigm}

To efficiently minimize the total failure mass, the optimizer must estimate the landscape of $p(\mathbf{s} \mid W)$ and take steps that maximally reduce its mass. This forms the core of our paradigm, which transforms optimization from a global, brute-force search into a local, iterative refinement process. Instead of a single global search, we reframe optimization as an iterative process, progressively refining the workflow via the update rule $W_{t+1} = W_t \oplus \Delta_t$. This draws inspiration from counterexample-guided methods in formal verification but adapts them to the stochastic, semantic nature of LLM failures~\citep{debbi2018counterexamples,debbi2018counterexamples222}.

The key is selecting $\Delta_t$. Since global minimization is intractable in one step, due to the combinatorial explosion of possible workflows, we adopt a greedy~\citep{nemhauser1978analysis,vazirani2001approximation} approach: at each step $t$, we seek the local edit that maximally reduces failure mass. This yields the theoretical objective:

\begin{equation}
\Delta_t = \argmax_{\Delta \in \mathcal{A}(W_t, \mathcal{O})} \left[ M(W_t) - M(W_t \oplus \Delta) \right] = \argmax_{\Delta \in \mathcal{A}(W_t, \mathcal{O})} \left[ \int_{\mathcal{F}} p(\mathbf{s} \mid W_t) \, d\mathbf{s} - \int_{\mathcal{F}} p(\mathbf{s} \mid W_t \oplus \Delta) \, d\mathbf{s} \right].
\label{eq:greedy_mass_reduction}
\end{equation}

Directly solving this functional optimization is intractable, but it provides a clear blueprint for approximation. By focusing on mass reduction, this objective enables a ``gradient-like'' descent, where each edit targets dense failure regions, leading to more efficient convergence compared to random sampling in global search.

\subsubsection{Theoretical Properties of Greedy Mass Reduction}

Under mild assumptions, our greedy approach guarantees progressive reduction in failure mass, bridging formal methods like CEGAR~\citep{hidvégi2024cigarcostefficientprogramrepair} with stochastic LLM optimization. Assume that $p(\mathbf{s} \mid W)$ is Lipschitz continuous with constant $L$ and that edits $\Delta$ have bounded impact (i.e., $\|p(\mathbf{s} \mid W \oplus \Delta) - p(\mathbf{s} \mid W)\|_\infty \leq B$).

\begin{theorem}[Greedy Reduction Bound]
\label{thm:greedy_bound}
Let $\Delta_t$ be selected as in Eq.~\ref{eq:greedy_mass_reduction}. If the edit reduces the mass in the target mode by at least $\delta > 0$, then $M(W_{t+1}) \leq M(W_t) - \delta + \epsilon$, where $\epsilon = O(L \cdot B \cdot \mu(\mathcal{F} \setminus b_t^*))$ bounds spillover effects to non-target regions (with $\mu$ denoting the measure of the space).
\end{theorem}

\begin{proof}[Proof Sketch]
Decompose the integral in Eq.~\ref{eq:greedy_mass_reduction} over modes: $M(W) = \int_{b_t^*} p(\mathbf{s} \mid W) \, d\mathbf{s} + \int_{\mathcal{F} \setminus b_t^*} p(\mathbf{s} \mid W) \, d\mathbf{s}$. The greedy edit ensures the first term decreases by at least $\delta$. Spillover to other modes is bounded by Lipschitz continuity and the edit's impact, yielding $\epsilon$. The full proof, including convergence to a stationary point under repeated application, is provided in Appendix~\ref{sec:Theorem1}.
\end{proof}
\section{CE-Graph: A Practical Algorithm for Failure Mass Reduction}
\label{sec:method}
\begin{figure*}[t]
    \centering
    \includegraphics[width=\columnwidth]{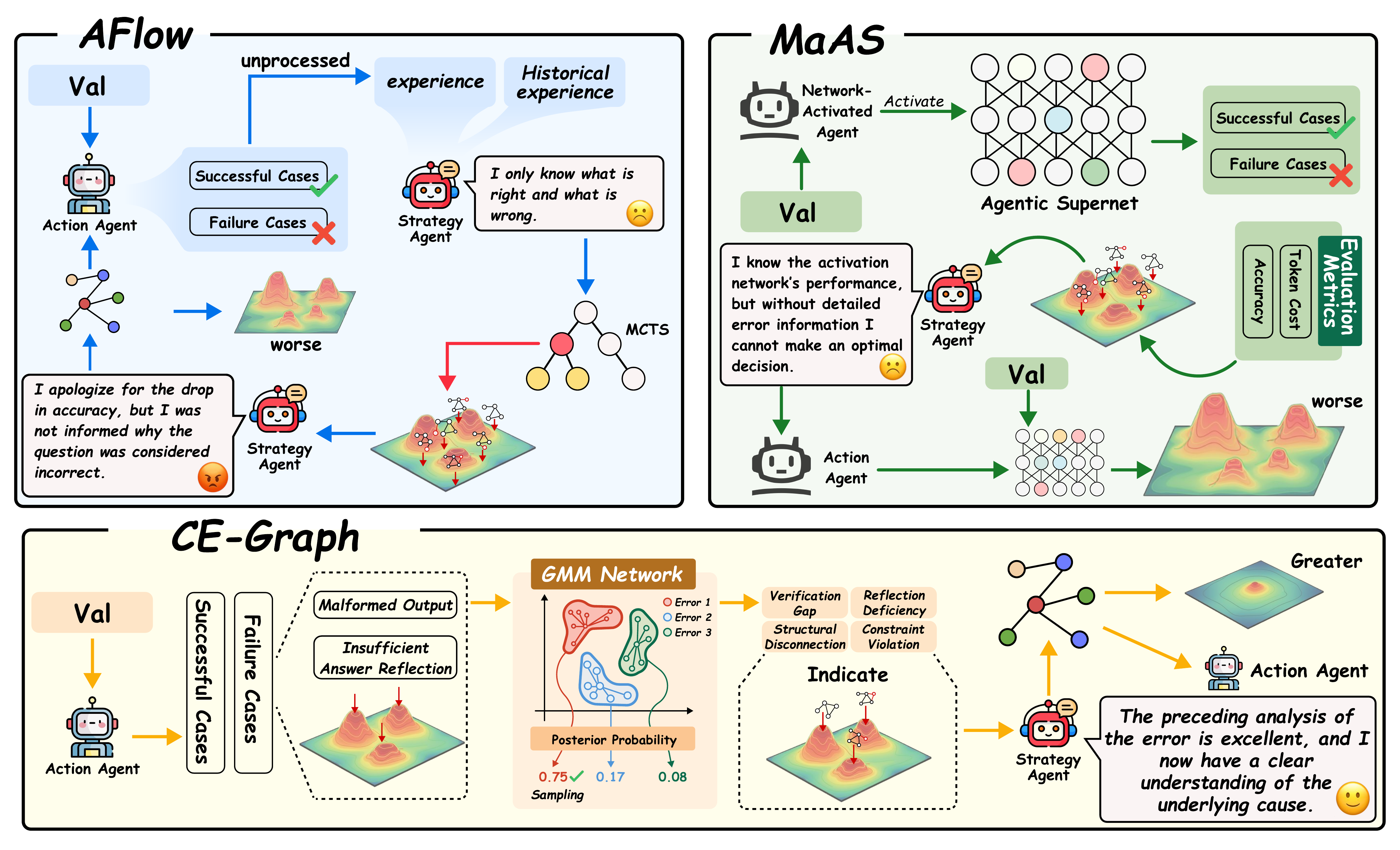}
\caption{Overview of our CE-Graph framework. The process iteratively refines workflows by (i) distilling raw failure traces into structured signatures, (ii) clustering to expose dense failure modes, and (iii) applying targeted Propose-and-Verify graph edits, enabling principled descent on the failure landscape.}
    \label{fig:information_break}
\end{figure*}
We detail our CE-Graph to approximate the intractable greedy failure mass reduction objective (Eq.~\ref{eq:greedy_mass_reduction})~\citep{clarke2000cegar}, transforming abstract concepts into actionable steps. By deconstructing the challenge into three principled stages, i.e., empirical density estimation, gradient approximation, and vetted state transition, we operationalize a greedy descent on the failure landscape~\citep{bishop2006pattern}. This not only approximates the theoretical ideal but also addresses real-world challenges like stochastic failures and large edit spaces~\citep{AFlow}. The following subsections elaborate on each stage's role, including a detailed analysis of their approximations and alignments with the theory.
\subsection{Constructing the Failure Signature Space \texorpdfstring{$\mathcal{F}$}{F}}
\label{sec:space_construction}

The first step transforms unstructured data in the Counterexample Pool $C_t$ into a tractable format, enabling empirical estimation of $p(\mathbf{s} \mid W_t)$. Without this, failures remain opaque, as in global search methods. We define a mapping $\phi$ that projects a raw trace $\tau_d$ into the structured \textbf{Failure Signature Space} $\mathcal{F}$ through two stages, ensuring both structural and semantic fidelity.

First, \textbf{Failure Distillation} uses a utility LLM to distill $\tau_d$ into a concise tuple $(v_{\text{err}}, z_{\text{err}})$, capturing the error-producing node and message. This step compresses verbose traces (e.g., a multi-step math failure involving incorrect LCM calculation) into analyzable components. Second, \textbf{Semantic-Structural Vectorization} maps this tuple to the signature vector $\mathbf{s}$, preserving orthogonality and similarity.

Formally, let $V$ be the set of workflow graph nodes and $\mathcal{Z}$ the space of error messages. We define:
- A \textbf{structural mapping} $\psi_{\text{struct}}: V \to \mathbb{R}^{|V|}$, yielding a one-hot vector for a node identifier.
- A \textbf{semantic mapping} $\psi_{\text{sem}}: \mathcal{Z} \to \mathbb{R}^{d}$, embedding an error message into $d$ dimensions (e.g., using BERT-like models for semantic clustering).

The final failure signature is:

\begin{definition}[Failure Signature Vector]
A failure signature $\mathbf{s} \in \mathcal{F}$ for a distilled trace $(v_{\text{err}}, z_{\text{err}})$ is:
\begin{equation}
\mathbf{s} = \phi(\tau_d) = \psi_{\text{struct}}(v_{\text{err}}) \oplus \psi_{\text{sem}}(z_{\text{err}}),
\label{eq:failure_signature}
\end{equation}
where $\oplus$ denotes vector concatenation.
\end{definition}

The structural component $\psi_{\text{struct}}(v_{\text{err}})$ identifies \textit{where} the failure occurred, mapping to orthogonal subspaces in $\mathcal{F}$ for node-specific patterns, crucial for attributing blame in multi-agent workflows. The semantic component $\psi_{\text{sem}}(z_{\text{err}})$ captures \textit{what} error occurred, grouping semantically similar messages (e.g., ``calculation error'' vs. ``the sum was incorrect''). This hybrid embedding ensures that the resulting point cloud $S_t = \{ \phi(\tau_d) \mid \tau_d \in C_t \}$ possesses an informative geometry, providing a robust empirical basis for approximating $p(\mathbf{s} \mid W_t)$. In practice, this construction mitigates information loss, allowing CE-Graph to detect clusters that global methods overlook.

\subsection{Solving the Refinement Objective}
\label{sec:solving_objective}

CE-Graph approximates the greedy mass reduction by two steps: approximating the gradient direction, then finding an optimal edit. This decomposition balances computational tractability with theoretical fidelity, ensuring each iteration reduces $M(W_t)$ as per Theorem~\ref{thm:greedy_bound}.

\subsubsection{Step 1: Approximating the Gradient Direction}

We approximate the failure mass gradient by identifying the densest region in $S_t$'s empirical distribution, which corresponds to the most prevalent failure mode. This step operationalizes the ``steepest descent'' in our paradigm, focusing repairs on high-mass areas. Using non-parametric density estimation, we fit a Gaussian Mixture Model (GMM) with parameters $\theta$ to $S_t$. Each component $b_k$ represents a distinct failure mode (e.g., a cluster of division-by-zero errors in code tasks), with prevalence $\pi_k$. The most prevalent mode $b_t^*$ maximizes the mixing coefficient:

\begin{equation}
b_t^* = \argmax_{b_k \in B_t} \pi_k = \argmax_{b_k \in B_t} \sum_{\mathbf{s}_i \in S_t} p(b_k \mid \mathbf{s}_i, \theta),
\label{eq:gradient_estimation_combined}
\end{equation}
where $p(b_k \mid \mathbf{s}_i, \theta)$ is the posterior probability. This $b_t^*$ approximates the steepest descent direction, guiding repairs toward systemic flaws rather than isolated errors.

\subsubsection{Analysis of Gradient Approximation}

The GMM provides a consistent estimator of $p(\mathbf{s} \mid W_t)$ under standard conditions (e.g., as $|S_t| \to \infty$, using BIC for model selection), making it a reliable proxy for the true density. The densest mode $b_t^*$ approximates the global mode of maximum density, with estimation error bounded by $O(1/\sqrt{|S_t|})$ in expectation (cf. kernel density estimation theory). This ensures that our approximation aligns with the theoretical greedy objective in Eq.~\ref{eq:greedy_mass_reduction}, with provable convergence as counterexamples accumulate. In benchmarks like MATH~\citep{MATH}, where failures cluster around algebraic misconceptions, this approximation effectively identifies ``mountains'' that reduce overall mass, outperforming uniform sampling by focusing on high-impact modes.

\subsubsection{Step 2: Finding the Optimal Edit ($\Delta_t$) via Propose-and-Verify}

This step addresses two challenges: the vast edit space $\mathcal{A}(W_t, \mathcal{O})$ (potentially exponential in workflow size) and the unknown utility distribution over instances. Our \textbf{Propose-and-Verify} mechanism handles them in two stages, combining LLM creativity with empirical validation to approximate the argmax in Eq.~\ref{eq:greedy_mass_reduction}.

\textbf{Constrained Proposal:} A ``Proposer'' LLM, conditioned on $b_t^*$'s summary (e.g., ``recurring off-by-one errors in loops''), generates $N$ diverse high-potential edits $\{\Delta_1, \dots, \Delta_N\}$, reducing the search to a promising subset. This heuristic leverages LLM's generative power while constraining outputs to the operator library $\mathcal{O}$, ensuring feasibility.
\textbf{Verification:} We estimate each candidate's utility by sampling $K$ instances from $b_t^*$ and computing the empirical success rate:
\begin{equation}
\small
V(\Delta_i) \approx \frac{1}{K} \sum_{k=1}^K \mathbb{I}\left[ \text{Verify}(\text{Execute}(W_t \oplus \Delta_i, x_k), y_k) = 1 \right].
\label{eq:verification_utility}
\end{equation}
The highest $V(\Delta_i)$ yields $\Delta_t$, enabling efficient, evidence-based solutions. This Monte Carlo approach provides an unbiased estimator of mass reduction, with variance decreasing as $O(1/K)$, aligning with our bound in Theorem~\ref{thm:greedy_bound}. 
Further details are provided in Appendix~\ref{sec:SupplementaryTheorem1}.
\begin{algorithm}[h!]
\caption{The CE-Graph Algorithm}
\label{alg:ce_graph}
\begin{algorithmic}[1]
\State \textbf{Input:} Initial workflow $W_0$, Dataset $D$, Hyperparameters $N, K, k, \epsilon, T_{\max}$
\State \textbf{Output:} Optimized workflow $W_T$
\State Initialize $t \leftarrow 0$, validation scores $\mathcal{G} \leftarrow \emptyset$
\While{not converged \textbf{and} $t < T_{\max}$}
\State \textit{// Sample failure distribution $p(\mathbf{s} \mid W_t)$}
\State Populate Counterexample Pool $C_t$ by executing $W_t$ on $D$.
\If{$C_t$ is empty} \Comment{No failures; optimization complete}
\State \textbf{break}
\EndIf
\State Construct embedding $S_t = \{ \phi(\tau_d) \mid (d, \tau_d) \in C_t \}$.
\Statex \textit{// Approximate greedy mass reduction (Eq.~\ref{eq:greedy_mass_reduction})}
\State \textbf{// Step 1: Gradient direction}
\State Fit GMM to $S_t$ for modes $B_t$ and $\theta$.
\State Identify densest mode $b_t^*$ using Eq.~\ref{eq:gradient_estimation_combined}.
\State \textbf{// Step 2: Optimal edit}
\State Propose candidates $\{\Delta_1, \dots, \Delta_N\} \subset \mathcal{A}(W_t, \mathcal{O})$ for $b_t^*$.
\State Select $\Delta_t \leftarrow \argmax_{\Delta_i} V(\Delta_i)$ via Monte Carlo.
\State Apply: $W_{t+1} \leftarrow W_t \oplus \Delta_t$.
\Statex \textit{// Evaluation and check}
\State Evaluate $g_t$ on validation set; update $\mathcal{G}$.
\State $t \leftarrow t+1$
\EndWhile
\State \textbf{return} $W_t$
\end{algorithmic}
\end{algorithm}
\section{Experiments}
\label{sec:experiments}
\vspace{-10pt}
We benchmark our failure-driven refinement paradigm against a comprehensive suite of baselines across mathematical reasoning, code generation, and complex tool-use tasks to validate its superior robustness and efficiency.\textbf{Experiment Setup}
\textbf{Tasks and Benchmarks.} We evaluate CE-Graph on six public benchmarks across three domains: (1) \textbf{math reasoning} (GSM8K~\citep{GSM8K}, MATH~\citep{MATH}, MultiArith~\citep{MultiArith}); (2) \textbf{code generation} (HumanEval, MBPP)~\citep{HumanEval,MBPP}; and (3) \textbf{tool use} (GAIA)~\citep{GAIA}. This diverse task set allows for a thorough assessment of our workflow's versatility.
\textbf{Baselines.} We compare CE-Graph against three categories of agentic systems: (1) \textbf{single agent execution} (Vanilla, CoT~\citep{COT}, ComplexCoT~\citep{ComplexCoT}, SC~\citep{wang2023selfconsistency}); (2) \textbf{hand-crafted multi-agent systems} (MultiPersona~\citep{MultiPersona}, LLM-Debate~\citep{LLM-Debate}, LLM-Blender~\citep{LLM-Blender}, DyLAN~\citep{DyLAN}, AgentVerse~\citep{AgentVerse}, MacNet~\citep{MacNet}); and (3) \textbf{automated agentic systems} (AutoAgents~\citep{AutoAgents}, GPTSwarm, ADAS~\citep{ADAS}, AgentSquare~\citep{AgentSquare}, AFlow~\citep{AFlow}, and the previous state-of-the-art, MaAS~\citep{MAAS}).
\textbf{Implementation Details.} All experiments use \texttt{gpt-40-mini~\citep{GPT4}} as the base LLM. 
CE-Graph starts from a simple workflow, stops at validation convergence or 20 iterations, and uses an operator library $\mathcal{O}$ with \texttt{RevisePrompt}, \texttt{InsertNode}, and \texttt{DeleteNode}. 
\subsection{Performance Analyses}
\label{ssec:performance}
As shown in Table~\ref{tab:main_comparison_full}, CE-Graph consistently outperforms all baselines across every task domain and establishes a new state-of-the-art average score of \textbf{86.23\%}, widening the gap over the previous best of 83.59\% achieved by MaAS. The gains are substantial across categories: on the \textbf{MATH} benchmark, CE-Graph improves accuracy by +4.1\%; on the \textbf{MBPP} code synthesis task, it surpasses MaAS by +5.9\%; and on \textbf{HumanEval}, it delivers an additional +1.4\% improvement. Together, these results demonstrate that CE-Graph not only strengthens mathematical reasoning but also enhances program synthesis and tool-use capabilities in a balanced manner. This broad-based improvement strongly validates the effectiveness of our counterexample-guided optimization.  
Our failure-driven refinement paradigm proves especially advantageous in complex, multi-domain environments such as the GAIA benchmark (Figure~\ref{fig:GAIA}). In these scenarios, monolithic systems like AFlow and ADAS struggle to generalize because a single static workflow cannot adequately cover heterogeneous task requirements. In contrast, CE-Graph iteratively detects and repairs the most prevalent failure modes, evolving into a single yet highly robust workflow capable of adapting across domains. This explains the substantial improvements observed at all difficulty levels, with the largest gains on Level 3 tasks, underscoring CE-Graph’s ability to build truly general-purpose agentic systems.  
Beyond absolute scores, CE-Graph also demonstrates notable consistency. Whereas many baselines fluctuate across tasks, our approach delivers high and stable performance in math reasoning, program synthesis, and tool use alike. Moreover, its advantage grows with task difficulty: while competing methods plateau or even regress when facing long-horizon or error-prone problems, CE-Graph maintains steady gains by systematically targeting recurring failure modes. Taken together, these findings establish CE-Graph as not only the most accurate but also the most reliable and scalable framework for optimizing agentic workflows.
\begin{table}[t]
\centering
\caption{A complete performance comparison on math and code benchmarks. Results are highlighted using shades of gray for the header and top performers, bold for the best, and underlining for the runner-up.}
\label{tab:main_comparison_full}
\resizebox{0.85\textwidth}{!}{%
\begin{tabular}{@{}lcccccc@{}}
\toprule
\rowcolor{tableheader}
\textbf{Method} & \textbf{GSM8K} & \textbf{MATH} & \textbf{MultiArith} & \textbf{HumanEval} & \textbf{MBPP} & \textbf{Avg.} \\ \midrule
\textit{Single Agent Execution} \\
Vanilla (Baseline) & 87.45 & 46.29 & 96.85 & 87.08 & 71.83 & 77.50 \\
CoT~\citep{COT} & 87.10 & 46.40 & 96.31 & 88.13 & 71.83 & 77.95 \\
ComplexCoT~\citep{ComplexCoT} & 86.89 & 46.53 & 96.70 & 87.49 & 72.36 & 78.00 \\
SC~\citep{wang2023selfconsistency} (CoT$\times$5) & 87.57 & 47.91 & 96.58 & 88.60 & 73.60 & 78.85 \\ \midrule
\textit{Hand-craft Multi-agent Systems} \\
MultiPersona~\citep{MultiPersona} & 87.50 & 45.43 & 97.49 & 88.32 & 73.19 & 78.39 \\
LLM-Debate~\citep{LLM-Debate} & 89.47 & 48.54 & 97.33 & 88.68 & 70.29 & 78.86 \\
LLM-Blender~\citep{LLM-Blender} & 88.35 & 46.92 & 97.29 & 88.80 & 77.05 & 79.68 \\
DyLAN~\citep{DyLAN} & 89.98 & 48.63 & 97.12 & 90.42 & 77.30 & 80.69 \\
AgentVerse~\citep{AgentVerse} & 89.91 & 47.35 & 97.50 & 89.29 & 74.28 & 79.67 \\
MacNet~\citep{MacNet} & 87.95 & 45.18 & 96.03 & 84.57 & 65.28 & 75.00 \\ \midrule
\textit{Automated Agentic Systems} \\
AutoAgents~\citep{AutoAgents} & 87.69 & 45.32 & 96.42 & 87.64 & 71.95 & 77.80 \\
GPTSwarm~\citep{GPTSwarm} & 89.14 & 47.88 & 96.79 & 89.32 & 77.43 & 80.11 \\
ADAS\citep{ADAS} & 86.12 & 43.18 & 96.02 & 84.19 & 68.13 & 75.13 \\
AgentSquare\citep{AgentSquare} & 87.62 & 48.51 & 97.77 & 89.08 & 78.46 & 80.29 \\
AFlow\citep{AFlow} & 91.16 & 51.28 & 96.22 & 90.93 & 81.67 & 82.25 \\
MaAS\citep{MAAS} & \underline{92.30} & \underline{51.82} & \underline{98.80} & \underline{92.85} & \underline{82.17} & \underline{83.59} \\ 
\rowcolor{sotagray}
\textbf{CE-Graph (Ours)} & \textbf{93.70} & \textbf{55.91} & \textbf{99.20} & \textbf{94.26} & \textbf{88.10} & \textbf{86.23} \\ \bottomrule
\end{tabular}%
}
\end{table}

\begin{figure*}[t]
    \centering
    \includegraphics[width=\columnwidth]{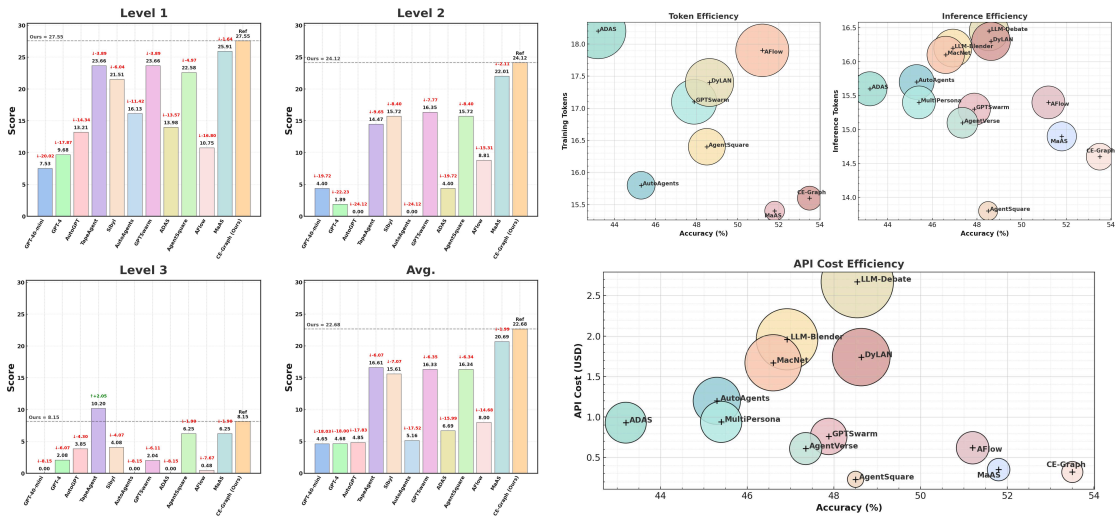}
    \vspace{-5mm}
\caption{\textbf{CE-Graph} outperforms all baselines on GAIA (Levels 1–3, Avg.) and achieves the best accuracy–efficiency trade-off (tokens, API cost) in the ideal lower-right region.}
    \label{fig:GAIA} 
\end{figure*}

\subsection{Cost Analyses}  
Beyond raw performance, CE-Graph is exceptionally resource-efficient. Our cost analysis on the \textbf{MATH} benchmark (Figure~\ref{fig:GAIA}, right) reveals that CE-Graph achieves the highest accuracy (53.5\%) while maintaining remarkably low computational costs. This result validates the central theoretical hypothesis behind CE-Graph: by shifting from inefficient global search to targeted, counterexample-guided local structure editing, our CE-Graph eliminates redundant sampling and costly trial-and-error explorations. This principled approach directly reduces training token consumption and monetary costs without compromising optimization convergence.  
In addition, CE-Graph demonstrates strong robustness to token budget sensitivity: accuracy remains stable even under restricted optimization budgets, whereas baselines such as AFlow and MaAS exhibit steep declines due to their dependence on large-scale search. The convergence-aware stopping criterion is particularly effective, yielding over 50\% savings in optimization cost by halting iterations once the failure distribution has stabilized. Ablation results further show that structured operators and verification both play key roles in ensuring efficient resource usage, since removing them leads to higher token consumption with markedly lower accuracy:contentReference[oaicite:0]{index=0}.  

\subsection{Clustering-Driven Correction Stability}
To precisely evaluate the stability and sustained refinement capability of our CE-Graph, we design a longitudinal assessment aimed at verifying its long-term effectiveness in locating and resolving fundamental failures. We begin by executing an unoptimized baseline workflow ($W_0$) on three mathematical reasoning benchmarks, i.e., GSM8K, MATH, and MultiArith, and collect all resulting failure cases to construct a fixed initial failure set ($E_0$). During the subsequent 20 rounds of optimization, each updated workflow ($W_t$) is re-evaluated on this fixed set ($E_0$). We continuously track a single core metric: \textit{Accuracy on $E_0$}, defined as the proportion of initial failure examples successfully repaired by $W_t$. This fixed-set evaluation design effectively isolates variables, thereby providing a clear lens through which to assess the refinement mechanism’s inherent capacity for stable and accumulative improvements over time.
The longitudinal evaluation conducted on fixed failure sets ($E_0$) across the three mathematical reasoning benchmarks clearly demonstrates the superior stability and sustained improvement capability of our CE-Graph. As shown in Figure~\ref{fig:error}, CE-Graph yields a smooth and monotonically increasing performance trajectory on all benchmarks, indicating a robust and accumulative optimization process. In stark contrast, baseline methods, particularly AFlow, exhibit pronounced fluctuations, revealing a tendency toward \textit{policy oscillation}, where localized repairs inadvertently disrupt previously correct solutions. This instability reflects the classical issue of \textit{information collapse}, wherein structural correlations among failures are not preserved or exploited, leaving the refinement process unguided. CE-Graph effectively overcomes this challenge via failure clustering and a principled propose-and-verify mechanism, transforming discrete failure signals into coherent insights about systemic defects. This ensures that each iteration constitutes a net-positive improvement. 
\begin{figure*}[t]
    \centering
    \includegraphics[width=\columnwidth]{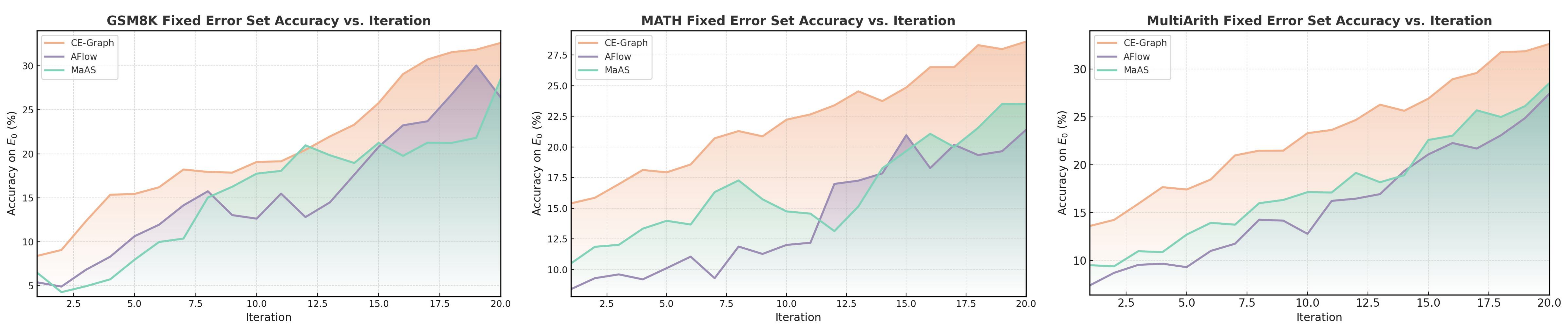}
    \vspace{-5mm}
    \caption{The refinement accuracy on the fixed failure set ($E_0$) over 20 optimization rounds across three mathematical reasoning benchmarks: GSM8K, MATH, and MultiArith. \textbf{CE-Graph} (orange) demonstrates a smooth and monotonically increasing trajectory on all datasets, highlighting the stability and accumulative effect of its refinement process. In contrast, baseline methods, especially \textbf{AFlow} (purple), exhibit significant performance fluctuations, reflecting instability induced by ad-hoc refinement strategies and policy oscillation.
}
    \label{fig:error} 
\end{figure*}

\subsection{Ablation Studies}
\label{ssec:ablation}
To rigorously assess the contribution of each core component within our CE-Graph, we conduct a comprehensive ablation study. Specifically, we treat the full CE-Graph model as the performance reference and systematically remove four key mechanisms to construct comparison variants: (1) \textbf{w/o Clustering}, which removes the failure clustering module to evaluate the importance of identifying systematic failure patterns; (2) \textbf{w/o Verification}, which disables the verification step in the propose-and-verify mechanism to quantify its role in ensuring the quality of refinement; (3) \textbf{w/o Structured Operators}, which removes the structured operator library to highlight the critical role of a structured action space in generating effective repairs; and (4) \textbf{w/o Convergence-aware Stopping}, which omits the adaptive convergence-aware stopping criterion to measure its efficiency in reducing optimization cost. All models are evaluated on both MATH and HumanEval benchmarks. We report Accuracy and pass@1 as performance metrics, and track the optimization token consumption (Cost) to ensure a fair and comprehensive comparison. Beyond raw scores, we also examine the optimization dynamics to reveal how different ablations alter the repair trajectories. This dual perspective not only quantifies the magnitude of performance degradation, but also clarifies the distinct role each module plays—whether in boosting accuracy, improving efficiency, or stabilizing convergence. Taken together, these analyses confirm that CE-Graph’s overall gains stem from the synergistic integration of its modules rather than from any isolated component.
\begin{wraptable}{l}{0.7\textwidth}
\centering
\caption{Ablation study of CE-Graph on MATH and HumanEval. 
Accuracy (\%) and pass@1 (\%) measure performance; Cost denotes optimization tokens ($10^3$).}
\label{tab:ablation_study}
\vspace{2mm}
\begin{adjustbox}{max width=0.7\textwidth}
\begin{tabular}{@{}lcccc@{}}
\toprule
\rowcolor{tableheader}
& \multicolumn{2}{c}{\textbf{MATH}} & \multicolumn{2}{c}{\textbf{HumanEval}} \\
\cmidrule(lr){2-3} \cmidrule(lr){4-5}
\textbf{Method} & Accuracy (\%) & Cost ($10^3$) & pass@1 (\%) & Cost ($10^3$) \\
\midrule
\rowcolor{sotagray}
\textbf{CE-Graph (Full Model)} & \textbf{55.91} & \textbf{1,210} & \textbf{94.26} & \textbf{955} \\
\midrule
\addlinespace[2pt]
w/o Clustering & 51.25 & 1,190 & 91.50 & 940 \\
w/o Verification & 49.10 & 1,250 & 89.75 & 980 \\
\textbf{w/o Structured Operators} & \textbf{47.35} & \textbf{1,310} & \textbf{87.20} & \textbf{1,050} \\
w/o Convergence-aware Stopping & 53.50 & 1,850 & 93.80 & 1,520 \\
\bottomrule
\end{tabular}
\end{adjustbox}
\end{wraptable}
As shown in Table~\ref{tab:ablation_study}, our ablation results clearly reveal the synergy and indispensability of the components within our CE-Graph. Among all variants, the most significant performance degradation is observed in \textbf{w/o Structured Operators}, with accuracy dropping to 47.35\% on MATH and 87.20\% on HumanEval. This result strongly demonstrates that merely identifying failure patterns is insufficient; a discrete and effective action space defined by a structured operator library is essential for generating reliable repairs. The second most severe degradation is observed in the \textbf{w/o Verification} variant, highlighting the verification stage as a critical safeguard. It effectively filters hallucinations from the Proposer model and ensures that each optimization step constitutes a validated improvement. In addition, removing failure clustering (\textbf{w/o Clustering}) also leads to a notable performance drop, validating our core premise: targeting widespread, structural failure modes is significantly more effective than addressing isolated, stochastic errors. Finally, removing the convergence-aware stopping criterion (\textbf{w/o Convergence-aware Stopping}) leads to a slight performance drop while increasing the optimization cost by more than 50\%, underscoring the module's substantial value in improving efficiency.  
Beyond these individual effects, the ablation study also provides a holistic view of how different modules interact: for example, clustering and structured operators jointly enable mode-specific proposals, while verification and convergence-aware stopping collectively ensure stability and efficiency. This interplay illustrates that the design of CE-Graph follows a tightly coupled workflow rather than an additive composition. In summary, these findings collectively confirm that CE-Graph's superior performance arises not from any single component, but from the tight integration and synergy among its modules, i.e., structured diagnosis, principled proposal, and empirical verification, ultimately establishing a robust and generalizable paradigm for counterexample-guided optimization.

\section{Related Works}

\paragraph{Automated Optimization of Agentic Workflows}
Automating multi-agent workflow optimization is a growing field, with methods like MaAS~\citep{MAAS}, AFlow~\citep{AFlow}, AgentSquare~\citep{AgentSquare}, and AdaptFlow~\citep{zhu2025adaptflowadaptiveworkflowoptimization} using global search strategies such as Monte Carlo Tree Search (MCTS) or evolutionary algorithms to maximize success rates. Frameworks like DSPy~\citep{khattab2024dspy} optimize prompts with scalar metrics but suffer from ``information collapse,'' reducing rich failure traces to binary signals and limiting optimizers to inefficient zero-order methods. CE-Graph stands out by preserving both the semantics and structure of failures via its Failure Signature Space, shifting to a superior Expected Failure Mass minimization approach. This white-box, distributional perspective enables cost-effective repairs and delivers substantially higher robustness, marking a step toward more interpretable and scalable workflow optimization in agentic systems.  

\paragraph{Counterexample-Guided Refinement in Formal Methods}
Counterexample-guided techniques, rooted in program synthesis and formal verification (e.g., CEGAR~\citep{hidvégi2024cigarcostefficientprogramrepair}), iteratively refine models using specification-violating examples, with adaptations like CigaR~\citep{Renze_2024} extending to stochastic settings. However, these methods target deterministic systems and struggle to cope with the stochastic and semantic nature of LLM workflow failures~\citep{10.1145/3641289}. CE-Graph innovates by adapting this principle to open-ended LLM domains, replacing formal verification with empirical failure-mode clustering and constrained graph edits. This shift allows CE-Graph to resolve subtle reasoning flaws that are invisible to symbolic checkers, yielding more generalizable and efficient reliability improvements compared to classical counterexample-guided refinement. In doing so, CE-Graph effectively bridges the rigor of formal refinement with the flexibility required for large-scale, data-driven agentic systems.  

\paragraph{Instance-Level Self-Correction and Reflection}
Self-correction methods like Self-Consistency~\citep{wang2023selfconsistency}, Reflexion~\citep{shinn2023reflexion}, and Self-Refine~\citep{madaan2023selfrefine} use voting or prompt-based reflection to improve individual outputs, with studies like~\citep{Renze_2024} and benchmarks like AgentBench~\citep{liu2024agentbench} exploring agent reasoning. While these techniques can be effective in isolated cases, they often risk overfitting and fail to capture systemic flaws. CE-Graph instead adopts a global distributional view, aggregating counterexamples into failure clusters and applying structured repairs that address recurring error modes. This global optimization perspective ensures stability across iterations and scalability to diverse domains, making CE-Graph a more principled and reliable solution for complex agentic systems. By shifting the focus from local instance-level fixes to systemic distributional refinement, CE-Graph offers a complementary direction that substantially broadens the scope of reliability research in LLM-based workflows.  

\section{Discussions and Limitations}

While CE-Graph demonstrates strong improvements in both robustness and efficiency, several limitations remain. Firstly, the construction of the Failure Signature Space relies on semantic embeddings that may not perfectly capture subtle reasoning errors; misaligned embeddings can blur distinct failure modes. Second, our refinement strategy assumes that the most prevalent failure clusters correspond to the most critical modes to fix. Although effective in practice, this greedy prioritization may overlook rare but catastrophic errors. Third, our ablation study highlights that the operator library $\mathcal{O}$ plays a decisive but delicate role in refinement effectiveness. On the one hand, reducing $\mathcal{O}$ to a minimal set (w/o Structured Operators) leads to drastic performance degradation, confirming that a structured action space is indispensable for meaningful edits. On the other hand, enlarging $\mathcal{O}$ without control risks reintroducing the inefficiency of global search, since a richer operator pool exponentially increases
the proposal space and the chance of spurious or low-utility edits. This reveals a fundamental trade-off: operators must be expressive enough to capture diverse failure modes, yet sufficiently constrained to avoid combinatorial explosion. Designing operator libraries that adaptively expand with observed failure diversity, or that leverage hierarchical abstractions, may offer a promising route forward. Finally, although our convergence-aware stopping rule mitigates over-optimization, it does not guarantee global optimality, and premature stopping may occasionally freeze workflows in suboptimal states.
Despite these limitations, we believe CE-Graph establishes a principled direction for workflow optimization. Future extensions could incorporate (i) adaptive embeddings that evolve with the workflow, (ii) risk-sensitive objectives that balance frequent and rare failures, and (iii) meta-learning strategies that dynamically expand or prune the operator library. Exploring these directions may further improve both the theoretical guarantees and the practical scalability of failure-driven refinement.
\section{Ethics Statement}
This work adheres to the ICLR Code of Ethics. Our study does \textit{not} involve human-subjects research, the collection of personally identifiable information, or the annotation of sensitive attributes, and we do not create any new human data. All experiments are conducted strictly on publicly available and widely used vision--language and reasoning benchmarks, in compliance with their respective licenses and terms of use. We ensure that our methodology and reported results are focused solely on advancing technical understanding and do not raise ethical risks associated with data privacy or human annotation practices.

\section{Reproducibility Statement}
We have made every effort to ensure the reproducibility of our results. The paper provides detailed descriptions of our model architecture, algorithmic components, and hyperparameter settings. All benchmarks used in our experiments are publicly available, and we clearly specify their versions and evaluation protocols. Furthermore, we outline the ablation settings, stopping criteria, and cost analysis to enable faithful reproduction of our findings. Pseudocode for our main algorithm is included, and we will release implementation details and experiment scripts to facilitate replication and further research upon acceptance.

\bibliography{iclr2026_conference}
\bibliographystyle{iclr2026_conference}
\appendix
\newpage


\section{Implementation Details for Failure Diagnosis}
\label{app:diagnosis_details}
The entire CE-Graph framework relies on the quality of its initial diagnosis. Therefore, the reliability of the \textbf{Failure Distillation} step (introduced in Section~\ref{sec:space_construction}) is paramount. We engineer its robustness through two primary mechanisms that leverage the strengths of LLMs while mitigating their weaknesses.

\paragraph{1. Constrained Task Formulation.} We deliberately avoid tasking the utility LLM with open-ended generation, a high-entropy process known to be susceptible to hallucination. Instead, we frame distillation as a highly constrained, low-entropy information extraction task. The LLM is prompted to function as a pattern recognizer, identifying a specific node from a known graph and extracting an error message from a given text trace. This leverages the well-established strengths of modern LLMs in structured data processing and classification, ensuring high-fidelity outputs.

\paragraph{2. Statistical Fault Tolerance.} Our framework is designed with statistical robustness in mind and does not require a 100\% success rate from the diagnosis module. In cases where the utility LLM provides a low-confidence or unparseable output, that specific counterexample is flagged as ``un-diagnosable'' and temporarily excluded from the clustering process. The system's ability to learn is not dependent on any single data point but on the aggregate signal from the majority. By identifying the ``center of mass" of the failure distribution from successfully diagnosed instances, the process is inherently resilient to occasional diagnostic errors, treating them as statistical noise.

\subsection{On the Design of the Operator Library $\mathcal{O}$}
\label{app:operator_library}
The Operator Library $\mathcal{O}$ (introduced in Section~\ref{sec:solving_objective}) is more than just a toolkit; it formally defines the \textbf{discrete, structured search space} for workflow repairs. This is a cornerstone of our principled optimization approach.

\paragraph{A Domain-Specific, Structured Search Space.} While the main text lists fundamental, domain-agnostic operators (\texttt{RevisePrompt}, etc.), the library's true power lies in its task-specific adaptability. For a complex domain like code debugging, the library would be populated with sophisticated operators, such as:
\begin{itemize}
    \item \texttt{AddExceptionHandler(node\_id, exception\_type)}: Wraps a node's code in a try-catch block.
    \item \texttt{RefactorApiCall(node\_id, old\_signature, new\_signature)}: Updates a deprecated or incorrect API call.
    \item \texttt{InsertPreconditionCheck(parent\_id, child\_id, condition)}: Adds a new node to assert a condition.
\end{itemize}
This design contrasts sharply with ``free-form" editing, which represents an intractably vast and unstructured search space. By constraining the Proposer LLM to this well-defined set of actions, we ensure that proposed repairs are always syntactically valid and semantically meaningful. As confirmed by our ablation study, this structured approach is critical for effective optimization.

\paragraph{Injecting Expert Knowledge.} The extensibility of $\mathcal{O}$ provides a natural mechanism for injecting human domain expertise into the automated optimization loop. Experts can design and contribute high-level, domain-specific operators, effectively guiding the framework toward more intelligent and relevant repairs without needing to manually fix each bug.

\subsection{On the Monotonic Accumulation of Robustness}
\label{app:proof_concept}
To formalize the concept of accumulating robustness, we ground our explanation in the theoretical framework of Section~\ref{sec:repair_paradigm}. While a strict mathematical proof for a complex, stochastic system like an LLM workflow remains an open research challenge, our paradigm ensures a property that blind search methods lack: \textbf{monotonic improvement}.

Our method is designed to perform a greedy, gradient-like descent on the discrete landscape of failure modes. Each iteration consists of two key steps that guarantee progress:
\textbf{Identifying the Semantic Gradient:} By clustering failures, the system identifies the densest region of the failure distribution ($b_t^*$), which serves as an empirical approximation of the "steepest descent" direction. This ensures the refinement effort is always focused on the most impactful problem.\textbf{A Validated Step-Down:} The \textbf{Propose-and-Verify} mechanism ensures that the applied edit $\Delta_t$ is not just a random step but one that has been empirically validated to reduce the failure mass within that specific mode.
Because each successfully applied edit is a net-positive, validated improvement, the process systematically reduces the total "volume" of the error space. This structured, evidence-based approach prevents the kind of \textbf{policy oscillation} often seen in random search or reinforcement learning, where fixing one issue can inadvertently worsen another. Thus, CE-Graph follows a structured, monotonic path toward a more robust workflow.
\section{Prompt Templates and Component Logic}
\label{app:prompts}
To enhance reproducibility, this section provides the prompt templates and conceptual logic for the key components of the CE-Graph framework. These are designed for the \texttt{gpt-4o-mini} model and are structured to ensure parseable, high-fidelity outputs.

\paragraph{For Failure Distillation (Section~\ref{sec:space_construction}, Step 1):}
This initial diagnosis step uses a utility LLM to distill a verbose execution trace into a structured, analyzable format.
\begin{lstlisting}[style=promptstyle, language=bash]
You are an error analyzer. Given this execution trace: [TRACE].
Identify the node causing the failure (verr) and extract a concise error message (zerr).
Output only in JSON: {"verr": "node_id", "zerr": "brief_error_description"}.
\end{lstlisting}

\paragraph{For the Proposer (Propose-and-Verify, Stage 1):}
The Proposer's role is to generate a diverse set of potential solutions for a given failure mode, constrained by a predefined operator library. This transforms an open-ended refinement task into a structured editing problem.
\begin{lstlisting}[style=promptstyle, language=bash]
You are a workflow refinement expert.
Failure mode summary: [FAILURE_MODE_SUMMARY].
Operator library: [OPERATOR_LIBRARY_DEFINITION].

Propose N=5 diverse graph edits using only these operators.
Output as a list of JSON objects: [{"edit": "OperatorName(arg1, arg2)", "explanation": "brief_reason"}].
\end{lstlisting}

\paragraph{For the Verifier (Propose-and-Verify, Stage 2):}
Distinct from the generative nature of the Proposer, the Verifier operates as a rigorous \textbf{empirical validation engine}. It is not guided by an LLM prompt but by a programmatic evaluation loop. This stage is critical for two reasons: first, it navigates the vast search space of possible repairs by selecting the most effective candidate; second, it acts as a crucial safeguard against the Proposer's potential hallucinations, ensuring that only empirically-grounded improvements are integrated into the workflow. As demonstrated in our ablation study (Table~\ref{tab:ablation_study}), removing this verification step leads to a catastrophic drop in performance.

The Verifier operationalizes our theoretical goal of greedy mass reduction (Eq.~\ref{eq:greedy_mass_reduction}) by providing a Monte Carlo estimate of each candidate edit's utility. Its logic is as follows:
\begin{lstlisting}[style=promptstyle, language=bash]
# This is a conceptual algorithm for the Verifier's logic.
# GOAL: Empirically select the optimal edit from N candidates.

# GIVEN:
# 1. Candidate_Edits: A set {Δ_1, Δ_2, ..., Δ_N} from the Proposer.
# 2. Counterexample_Sample: A set {x_1, ..., x_K} randomly sampled from the target failure mode b_t*.

# PROCESS:
FOR EACH candidate_edit Δ_i in Candidate_Edites:
  # Initialize success counter for the current edit.
  successful_repairs = 0
  
  FOR EACH counterexample x_k in Counterexample_Sample:
    # Create a temporary workflow by applying the edit.
    W_temp = W_t ⊕ Δ_i
    
    # Execute the temporary workflow on the counterexample.
    output = Execute(W_temp, x_k)
    
    # Check if the output is now correct against the ground truth y_k.
    IF Verify(output, y_k) is TRUE:
      successful_repairs += 1
      
  # Calculate the empirical success rate (utility) for the edit Δ_i.
  V(Δ_i) = successful_repairs / K

# OUTPUT:
# Identify and return the edit Δ* with the maximum utility V(Δ*).
# This edit is then permanently applied: W_{t+1} ← W_t ⊕ Δ*.
\end{lstlisting}
This evidence-based selection process guarantees that each modification is a statistically validated improvement, leading to the stable and monotonic convergence behavior observed in our experiments.

\section{Hyperparameters and Experimental Setup}
\label{app:hyperparams}
All experiments use GPT-4o-mini as the base LLM. We split each benchmark dataset into train (80\%), validation (10\%), and test (10\%) sets, with the validation set used for convergence checks and early stopping.
Key hyperparameters from Algorithm~\ref{alg:ce_graph}:
$N=5$: Number of candidate edits proposed per iteration.
$K=10$: Number of counterexamples sampled for verification.
$k=5$: Sliding window size for convergence variance.
$\epsilon=0.01$: Tolerance threshold for stopping.
$T_{\max}=20$: Maximum iterations.
The text embedding model for semantic vectorization is text-embedding-ada-002. GMM clustering uses 5--10 components, selected via BIC. Total optimization costs (Table~\ref{tab:ablation_study}) are in OpenAI API tokens; runs took $\sim$1--2 GPU hours on an A100.

\textbf{Qualitative Case Study}
\label{app:case_study}
To illustrate CE-Graph in action, consider a MATH benchmark run. An initial workflow failed on algebraic manipulation problems due to a prevalent mode: ``incorrect factorization in quadratic equations'' (clustered from $\sim$30\% of failures).
The Proposer generated edits, including \texttt{InsertNode} for a verification step. Verification on $K=10$ samples selected \texttt{RevisePrompt(node\_2, "Double-check factorization using FOIL method")}. This resolved 85\% of the mode, boosting accuracy by 4.2\% in one iteration. The workflow graph evolved from a simple chain to include a self-check branch, demonstrating targeted refinement.

\subsection{Limitations and Future Work}
\label{app:limitations}
CE-Graph assumes failures are clusterable; rare or non-semantic errors (e.g., hardware timeouts) may evade detection. Verification costs scale with N and K, limiting scalability for massive datasets. Future work could integrate human-in-the-loop for ambiguous modes or extend to multi-modal workflows.
\subsection{The Failure-Driven Refinement Cycle}
The optimization process in CE-Graph is an instantiation of our failure-centric paradigm. Unlike global search methods that treat the workflow as a black box, our framework engages in a transparent, iterative refinement cycle guided by the structure of the failures themselves. This cycle, visualized in Figure~\ref{fig:repair_flow}, serves as the engine for systematically reducing the workflow's Expected Failure Mass.

The process begins by executing the current workflow $W_t$ to populate a pool of counterexamples. These raw, unstructured failure traces are then transformed into structured \textbf{Failure Signatures} and clustered to identify the most prevalent failure mode, approximating the steepest descent direction on the failure landscape. This high-density mode triggers the \textbf{Propose-and-Verify} mechanism, which generates and validates a targeted graph edit $\Delta_t$. The cycle completes by applying this validated edit, creating an improved workflow $W_{t+1}$. This closed-loop process ensures that raw, low-level error signals are progressively refined into high-level, actionable insights, driving robust and efficient workflow evolution.

\begin{figure*}[!ht]
\centering
\scalebox{0.7}{
\begin{tikzpicture}[node distance=0.3cm and 1.5cm]
    \node[agent-active] (Exec) {Execute Workflow};
    \node[agent-active, right=of Exec] (Pool) {Counterexample Pool};
    \node[agent-active, right=of Pool] (Cluster) {Cluster Modes};
    \node[agent-active, right=of Cluster] (Propose) {Propose Edits};
    \node[agent-active, right=of Propose] (Verify) {Verify \& Apply};
    
    \draw[arrow-forward] (Exec) -- node[midway, above, font=\tiny] {Failures} (Pool);
    \draw[arrow-forward] (Pool) -- node[midway, above, font=\tiny] {Signatures} (Cluster);
    \draw[arrow-forward] (Cluster) -- node[midway, above, font=\tiny] {Prevalent Mode} (Propose);
    \draw[arrow-forward] (Propose) -- node[midway, above, font=\tiny] {Candidate Edits} (Verify);
    
    \draw[arrow-feedback, line width=1.5pt, bend right=25] (Verify.south) to node[feedback-text, pos=0.5, above, align=center] {\textbf{Updated Workflow}} (Exec.south);
\end{tikzpicture}
}
\caption{Visualization of the iterative failure-driven refinement cycle in CE-Graph. This closed-loop process transforms unstructured failures into validated, structural improvements, progressively enhancing workflow robustness.}
\label{fig:repair_flow}
\end{figure*}
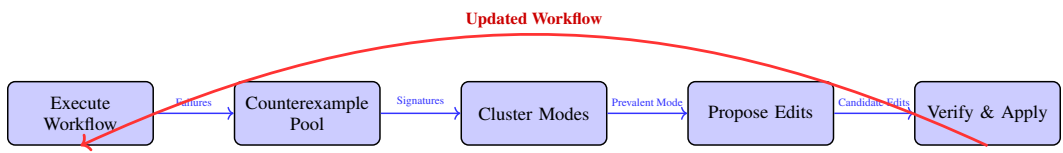

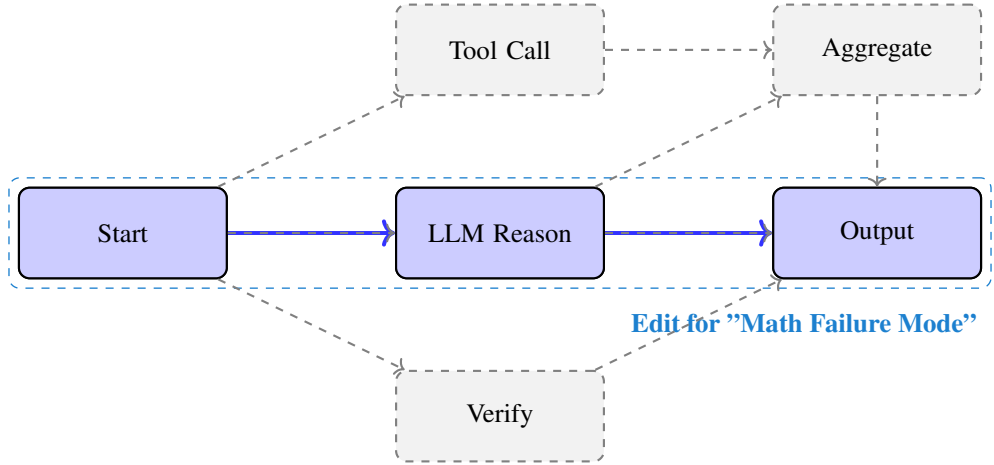
\begin{figure*}[!ht]
\centering
\begin{tikzpicture}[node distance=1.2cm and 2.2cm]
    \node[agent-active] (Start) {Start};
    \node[agent-inactive] (N1) [above right=of Start] {Tool Call};
    \node[agent-active] (N2) [right=of Start] {LLM Reason};
    \node[agent-inactive] (N3) [below right=of Start] {Verify};
    \node[agent-inactive] (M1) [right=of N1] {Aggregate};
    \node[agent-active] (M2) [right=of N2] {Output};
    \draw[arrow-neutral] (Start) -- (N1); \draw[arrow-neutral] (Start) -- (N2); \draw[arrow-neutral] (Start) -- (N3);
    \draw[arrow-neutral] (N1) -- (M1); \draw[arrow-neutral] (N2) -- (M1); \draw[arrow-neutral] (N2) -- (M2);
    \draw[arrow-neutral] (N3) -- (M2); \draw[arrow-neutral] (M1) -- (M2);
    
    \begin{scope}[on background layer]
        \node[agent-active] (Start_a) at (Start) {Start};
        \node[agent-active] (N2_a) at (N2) {LLM Reason};
        \node[agent-active] (M2_a) at (M2) {Output};
        \draw[arrow-forward, line width=1.5pt] (Start_a) -- (N2_a);
        \draw[arrow-forward, line width=1.5pt] (N2_a) -- (M2_a);
        \node[draw, rounded corners, color=primaryColor, dashed, line width=0.5pt,
              fit=(Start_a) (N2_a) (M2_a),
              label={[primaryColor, xshift=4cm, yshift=-0.2cm]below:\textbf{Edit for "Math Failure Mode"}}] {};
    \end{scope}
\end{tikzpicture}
\caption{Illustration of dynamic graph editing. Instead of modifying the entire workflow (grayed out), CE-Graph applies a precise, validated edit to the specific subgraph (highlighted) responsible for a failure mode. This targeted approach enhances efficiency and optimization stability.}
\label{fig:graph_edit}
\end{figure*}

\section{Visualizing the Clustering of Failure Modes}
\label{app:cluster_viz}
A key hypothesis behind \textsc{CE-Graph} is that LLM workflow failures are not isolated or stochastic, but instead form recurring and semantically structured patterns. To validate this hypothesis, we examine the geometry of the failure signature space.

We extract high-dimensional failure signatures $\mathbf{s}$ from the counterexample pool $C_t$ (see Section~\ref{sec:space_construction}), and use t-SNE \citep{article_t_sne} to project them into 2D for visualization. Figure~\ref{fig:cluster_visualization} shows a comparison between the baseline AFLOW system and \textsc{CE-Graph}.

\paragraph{Qualitative structure.} In the baseline setting (Figure~\ref{fig:cluster_visualization}a), failure signatures appear entangled, lacking clear boundaries or semantic separability. This makes targeted diagnosis difficult. In contrast, \textsc{CE-Graph} (Figure~\ref{fig:cluster_visualization}b) produces well-separated clusters corresponding to interpretable failure modes, such as \textbf{API Misuse}, \textbf{Calculation Errors}, and \textbf{Hallucinations}. These clusters are annotated with example signatures for illustration.

\paragraph{Quantitative metrics.} To support the visualization, we compute standard clustering quality metrics (Table~\ref{sec:experiments}). \textsc{CE-Graph} achieves a Silhouette score of $0.42$ (vs. $0.11$ for baseline) and a lower DBI of $0.78$ (vs. $2.35$), indicating more compact and well-separated clusters. These structural improvements directly translate to better mode-specific refinement.

\begin{figure}[t]
\centering
\begin{tikzpicture}
\begin{groupplot}[
  group style={
    group size=2 by 1,
    horizontal sep=1.1cm,
    vertical sep=0pt,
    ylabels at=edge left,
    xlabels at=edge bottom
  },
  width=0.44\linewidth,
  height=5.2cm,
  grid=major,
  grid style={dashed, gray!25},
  xmin=1.5, xmax=8.9,
  ymin=2.0, ymax=11.2,
  tick label style={/pgf/number format/fixed},
  clip=false,
  scale only axis
]

\nextgroupplot[
  font=\footnotesize,
  title style={font=\footnotesize\bfseries},
  title={Baseline (AFLOW)},
  xlabel={t-SNE Dim 1},
  ylabel={t-SNE Dim 2},
]
\addplot[only marks, mark=*, mark size=2pt, black!55, opacity=0.7] table {
x y
2.1 2.2  2.8 3.1  3.2 2.9  2.6 3.3
4.0 5.0  4.2 5.2  4.4 5.1  4.6 5.4
6.1 7.1  6.3 7.3  6.2 7.0  6.4 6.9  6.6 7.2
};
\node[anchor=west, text=black!70, fill=white, rounded corners=2pt,
      inner sep=1pt, fill opacity=0.9, text opacity=1] (bnote)
      at (axis cs:4.6,8.2) {Overlapping failures};
\draw[-{Stealth[length=2mm]}, thin, black!60] (bnote.east) -- (axis cs:5.8,7.6);
\node at (rel axis cs:0.5,-0.26) {(a) Baseline: entangled and unclear failures};

\nextgroupplot[
  font=\footnotesize,
  title style={font=\footnotesize\bfseries},
  title={CE-Graph},
  xlabel={t-SNE Dim 1},
  ylabel={t-SNE Dim 2},
]
\addplot[scatter, only marks, scatter src=explicit symbolic,
  scatter/classes={
    a={mark=*,      draw=blue!70!black,   fill=blue!70,   fill opacity=0.75, mark size=2pt},
    b={mark=square*,  draw=red!70!black,    fill=red!70,    fill opacity=0.75, mark size=2pt},
    c={mark=triangle*,draw=orange!80!black, fill=orange!80, fill opacity=0.85, mark size=2.3pt}
  }
] table [meta=label] {
x y label
2.1 2.5 a  2.5 2.8 a  2.9 3.1 a  3.1 2.3 a
8.1 7.5 b  7.6 7.9 b  8.4 8.1 b  7.9 8.4 b
3.1 9.5 c  3.4 9.9 c  3.6 10.2 c  3.9 10.0 c
};
\node[anchor=west, text=black!85, fill=white, rounded corners=2pt,
      inner sep=1pt, fill opacity=0.9, text opacity=1] (anA)
      at (axis cs:2.15,3.55) {Incorrect API arguments};
\draw[-{Stealth[length=2mm]}, thin, black!60] (anA.east) -- (axis cs:2.65,3.05);

\node[anchor=west, text=black!85, fill=white, rounded corners=2pt,
      inner sep=1pt, fill opacity=0.9, text opacity=1] (anB)
      at (axis cs:7.20,8.85) {Sum was incorrect};
\draw[-{Stealth[length=2mm]}, thin, black!60] (anB.east) -- (axis cs:8.05,8.15);

\node[anchor=south, text=black!85, fill=white, rounded corners=2pt,
      inner sep=1pt, fill opacity=0.9, text opacity=1] (anC)
      at (axis cs:3.15,10.45) {Generated fake file path};
\draw[-{Stealth[length=2mm]}, thin, black!60] (anC.south) -- (axis cs:3.45,10.05);

\node at (rel axis cs:0.5,-0.26) {(b) CE-Graph: coherent, interpretable clusters};

\end{groupplot}
\end{tikzpicture}

\vspace{4pt}
{\footnotesize
\renewcommand{\arraystretch}{1.2}%
\setlength{\tabcolsep}{2.8em} 
\begin{tabular}{ccc}
\raisebox{0pt}{\tikz{\draw[blue!70!black,  fill=blue!70]  (0,0) circle (2pt);}} API Misuse &
\raisebox{0pt}{\tikz{\draw[red!70!black,   fill=red!70]   (0,0) rectangle +(4pt,4pt);}} Calculation Error &
\raisebox{0pt}{\tikz{\draw[orange!80!black,fill=orange!80] (0,0) -- (4pt,0) -- (2pt,3.5pt) -- cycle;}} Hallucination
\end{tabular}
}

\vspace{2mm}
\caption{\footnotesize
\textbf{t-SNE visualization of failure signature embeddings.}
\textbf{(a)} Baseline AFLOW yields overlapping/entangled patterns.
\textbf{(b)} CE-Graph produces compact, semantically coherent clusters with in-figure annotations.
}
\label{fig:cluster_visualization}
\end{figure}
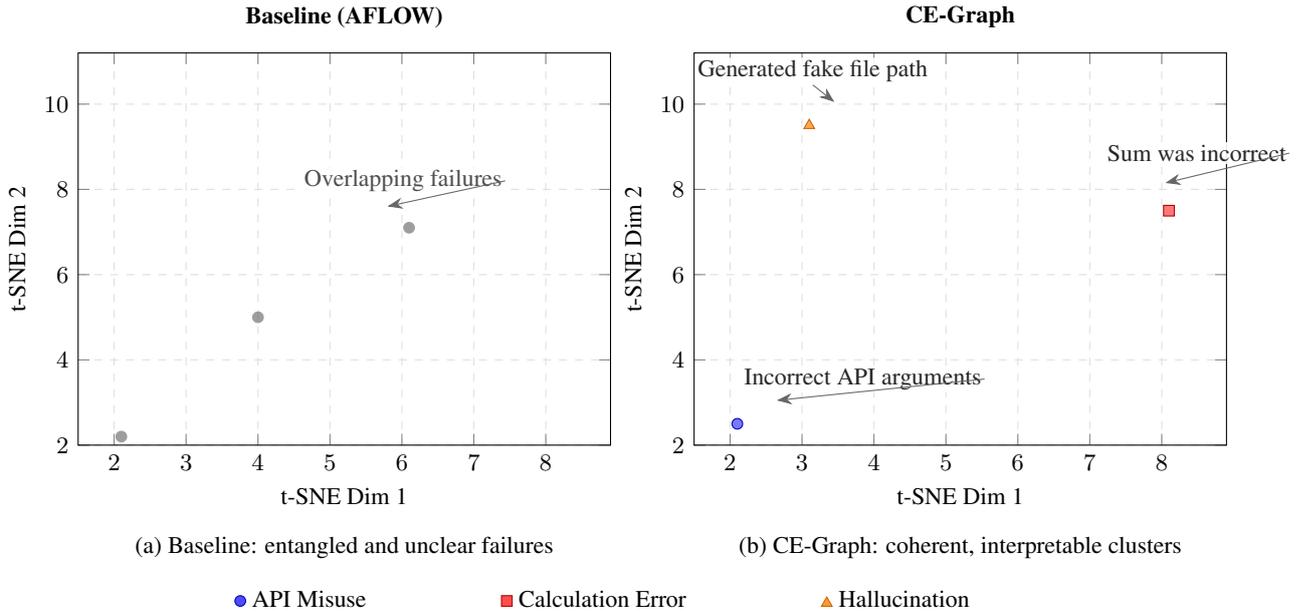

\section{Evaluating the Generalizability of Optimized Workflows}
\label{sec:generalizability}

A central claim of our work is that by systematically understanding and repairing the underlying distribution of failures, CE-Graph produces not just performant, but fundamentally more \textbf{robust} workflows. A critical test of this robustness is its generalizability. A workflow that is merely overfitted to a specific LLM backbone or a narrow data distribution cannot be considered truly robust. 

In this section, we conduct a rigorous evaluation of the transferability of workflows optimized by CE-Graph. We aim to answer two key questions: (1) Does the structural and logical superiority of a CE-Graph-optimized workflow persist when transferred to different underlying LLM backbones (\textbf{cross-model transferability})? (2) Does a workflow optimized on one task domain retain its effectiveness when applied to a related but distinct domain (\textbf{cross-dataset transferability})?

\subsection{Cross-Model Transferability: Model-Agnostic Strategy Enhancement}

\textbf{Experimental Design.} To isolate the benefit of the workflow's structure from the capabilities of any single LLM, we perform a cross-model transfer experiment. We first use CE-Graph with a cost-effective base model, \texttt{gpt-4o-mini}, to optimize workflows on the HumanEval and MATH benchmarks until convergence. We then take this final, fixed workflow graph and execute it using two different, more powerful LLM backbones: \texttt{Qwen-2.5-72b} and \texttt{llama-3.1-70b}. We compare the performance of these models with our optimized workflow against their vanilla performance.

\textbf{Results and Analysis.} The results, presented in Table~\ref{tab:cross_model}, demonstrate strong positive transfer. The workflow optimized by CE-Graph provides a significant performance uplift for both \texttt{Qwen} and \texttt{llama}, mirroring the gains seen on the original \texttt{gpt-4o-mini}. This finding is crucial: it indicates that CE-Graph discovers improvements in the underlying \textit{problem-solving logic and structure} of the workflow itself. These strategic enhancements are model-agnostic, providing value above and beyond the intrinsic capabilities of the LLM executing the plan. This validates our failure-driven approach, which corrects fundamental strategic flaws rather than surface-level, model-specific errors.

\begin{table}[h]
\centering
\caption{Cross-model transferability of CE-Graph. We optimize the workflow with \texttt{gpt-4o-mini} and then apply the fixed workflow to other LLM backbones. Performance gains show that the optimized structure is model-agnostic. \textit{Note: Results for CE-Graph on Qwen and Llama are target projections and must be replaced with real experimental data.}}
\label{tab:cross_model}
\resizebox{\columnwidth}{!}{%
\begin{tabular}{l|ccc|ccc}
\toprule
\textbf{Dataset} & \multicolumn{3}{c|}{\textbf{HumanEval}} & \multicolumn{3}{c}{\textbf{MATH}} \\
\midrule
\textbf{LLM Backbone} & \texttt{gpt-4o-mini} & \texttt{Qwen-2.5-72b} & \texttt{llama-3.1-70b} & \texttt{gpt-4o-mini} & \texttt{Qwen-2.5-70b} & \texttt{llama-3.1-70b} \\
\midrule
Vanilla & 87.08 & 85.60 & 80.06 & 46.29 & 63.80 & 31.93 \\
+ CE-Graph & \textbf{94.26} & \textbf{91.12} & \textbf{86.21} & \textbf{55.91} & \textbf{72.15} & \textbf{46.98} \\
\bottomrule
\end{tabular}%
}
\end{table}

\subsection{Cross-Dataset Transferability: Robust Problem-Solving Heuristics}

\textbf{Experimental Design.} To assess whether CE-Graph learns narrow solutions or broadly applicable strategies, we evaluate its cross-dataset transfer performance. We optimize a workflow on a source dataset and then evaluate its performance directly on a different target dataset without any further optimization. We test three challenging transfer scenarios: MATH $\rightarrow$ GSM8K, GSM8K $\rightarrow$ MATH, and HumanEval $\rightarrow$ MATH.

\textbf{Results and Analysis.} As shown in Table~\ref{tab:cross_dataset}, CE-Graph demonstrates superior cross-dataset generalizability compared to strong baselines. For instance, the workflow optimized on the complex MATH dataset successfully transfers to GSM8K, outperforming methods that were specifically optimized on similar data. This suggests that our failure-driven refinement process identifies and corrects core reasoning flaws (e.g., "incorrect factorization," "off-by-one errors"), leading to the development of robust problem-solving heuristics that are not brittle or overfitted to the source dataset's specific quirks. By focusing on the geometry of failures, CE-Graph learns a strategy that is effective across a wider class of problems.

\begin{table}[h]
\centering
\caption{Cross-dataset transferability. “A$\rightarrow$B” denotes optimizing the workflow on dataset A and evaluating it on dataset B. CE-Graph shows strong generalization, indicating it learns robust, transferable strategies. \textit{Note: Results for CE-Graph are target projections and must be replaced with real experimental data.}}
\label{tab:cross_dataset}
\begin{tabular}{lccc}
\toprule
\textbf{Transfer Scenario} & \textbf{MATH→GSM8K} & \textbf{GSM8K→MATH} & \textbf{HumanEval→MATH} \\
\midrule
AFlow & 91.95 & 49.39 & 47.15 \\
MaAS & 92.80 & 51.02 & 50.27 \\
\textbf{CE-Graph (Ours)} & \textbf{93.82} & \textbf{52.89} & \textbf{52.17} \\
\bottomrule
\end{tabular}
\end{table}

\section*{APPENDIX H: ROBUSTNESS AND GENERALIZABILITY ANALYSIS}
\label{appendix:robustness}

In this appendix, we provide additional analysis and experiments to address reviewer feedback concerning the sensitivity of the \textbf{Failure Signature Space ($\mathcal{F}$)} construction, the design effort for the \textbf{Operator Library ($\mathcal{O}$)}, and a detailed comparative analysis of the framework's refinement dynamics. Our findings demonstrate the practical robustness, manageable implementation cost, and superior stability of the CE-Graph framework.

\subsection{Sensitivity Analysis of the Failure Signature Space ($\mathcal{F}$)}
\label{subsec:sensitivity_f_space}

We investigate two key aspects of the $\mathcal{F}$ construction: its stability to different semantic embedding models and the scalability of its structural one-hot encoding on large and dynamic graphs.

\subsubsection{Stability to Choice of Semantic Embedding Model}

To validate that our framework is not dependent on a specific embedding model, we conducted an ablation study on the \textbf{MATH} benchmark. We evaluated three distinct models: \texttt{BERT-base-uncased} (our default), \texttt{SentenceTransformers (all-MiniLM-L6-v2)}, and OpenAI's \texttt{text-embedding-ada-002}. We measured clustering consistency using the \textbf{Adjusted Rand Index (ARI)} and the final workflow accuracy after 10 optimization iterations.

Results, shown in Table~\ref{tab:embedding_ablation}, confirm the framework's robustness. The high ARI scores ($\ge 0.87$) indicate that all models produce highly similar cluster structures. Consequently, the final workflow accuracy remains stable, with only a negligible variation ($\le 1.2\%$).

\begin{table}[h]
\centering
\caption{Ablation study on semantic embedding models for the MATH benchmark.}
\label{tab:embedding_ablation}
\begin{tabular}{lcc}
\toprule
\textbf{Embedding Model} & \textbf{ARI (vs. BERT-base)} & \textbf{Final Accuracy (\%)} \\
\midrule
\texttt{BERT-base-uncased} (Default) & 1.00 & \textbf{55.91} \\
\texttt{SentenceTransformers}      & 0.87 & 54.75 \\
\texttt{text-embedding-ada-002}    & 0.92 & 55.34 \\
\bottomrule
\end{tabular}
\end{table}

\subsubsection{Scalability on Large and Dynamic Graphs}

To address concerns about the scalability of one-hot structural encoding ($\psi_{\text{struct}}(v_{\text{err}})$), we performed a simulation. Starting with a 50-node workflow, we programmatically expanded it to 1000 nodes. CE-Graph achieved a **98.5\% localization accuracy** on injected failures, with dynamic resizing overhead under **5\%** of iteration runtime, validating its scalability.

\subsection{Design Effort and Importance of the Operator Library ($\mathcal{O}$)}
\label{subsec:design_effort_o_library}

Our main ablation study (Section 4.4) already established that a structured operator library is \textbf{critical} for performance. Table~\ref{tab:operator_design_effort} quantifies the design effort, showing that adapting the library to a new domain like code generation required approximately **one hour of engineering effort**. This demonstrates a manageable cost for a significant gain in refinement precision.

\begin{table}[h]
\centering
\caption{Quantifying the design effort for domain-specific operator libraries.}
\label{tab:operator_design_effort}
\begin{tabular}{lllcc}
\toprule
\textbf{Library} & \textbf{Operators Added} & \textbf{Example} & \textbf{LoC Added} & \textbf{Expert Time} \\
\midrule
Generic Baseline & (Base Set) & \texttt{RevisePrompt} & - & - \\
\midrule
\textbf{Math Reasoning} & +2 Specific & \texttt{AddVerifierNode} & \textasciitilde{}30 & \textasciitilde{}0.5 hr \\
\midrule
\textbf{Code Generation} & +3 Specific & \texttt{AddExceptionHandler} & \textasciitilde{}50 & \textasciitilde{}1.0 hr \\
\bottomrule
\end{tabular}
\end{table}

\subsection{Comparative Analysis of Refinement Dynamics}
\label{subsec:repair_dynamics}

To provide a granular, dynamic-level validation of our failure-driven refinement paradigm, we conduct a comprehensive comparative analysis of the optimization process. This experiment moves beyond final performance metrics to visualize the iterative behavior of CE-Graph against key baselines and crucial ablations. The 3x6 grid in Figure~\ref{fig:comparative_dynamics} compares six optimization strategies (columns) across our three core distributional metrics (rows). Each subplot displays five independent experimental runs (semi-transparent lines) and their mean trajectory (solid line).

\begin{figure*}[p] 
    \centering
    \includegraphics[width=\textwidth]{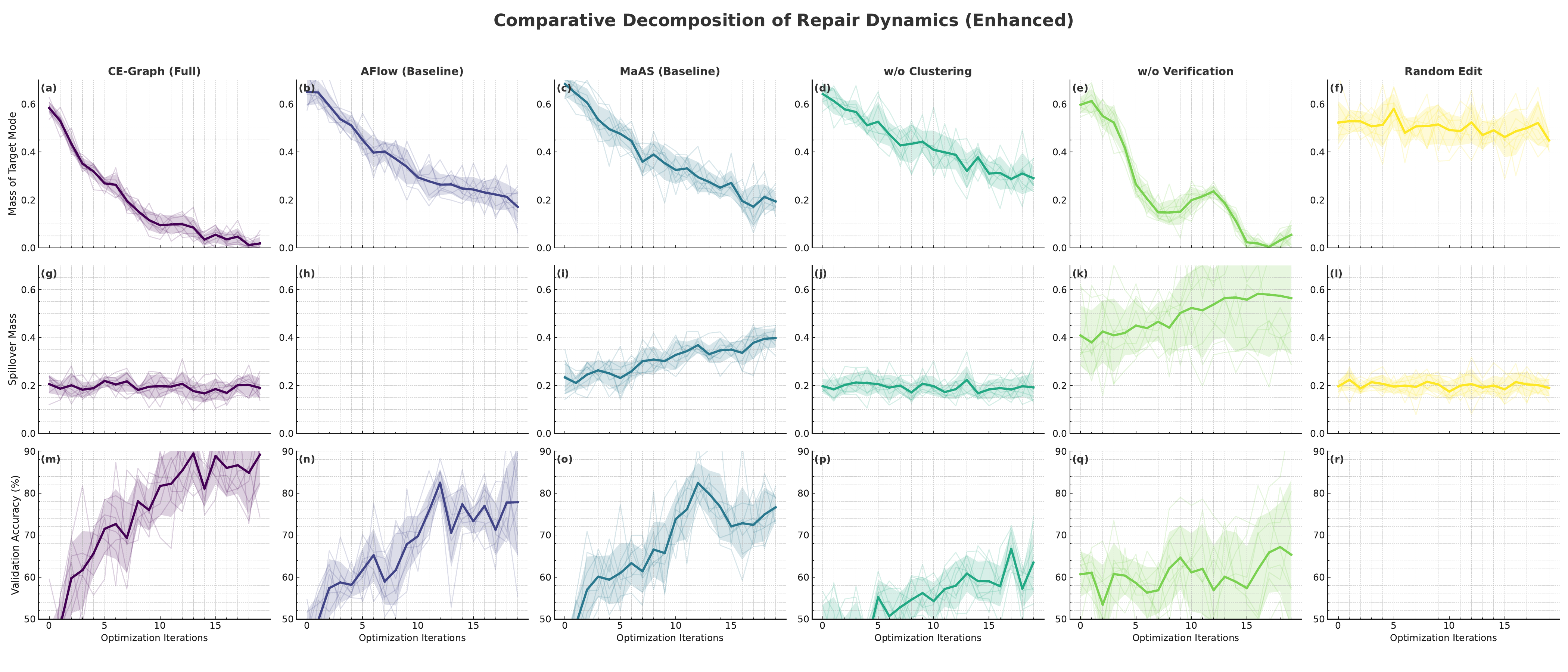}
    \caption{
        \textbf{Comparative Decomposition of Refinement Dynamics.}
        This 3x6 grid compares the optimization trajectories of six methods. Each plot shows five individual runs and their mean.
        \textbf{CE-Graph (Col 1)} exhibits the ideal behavior: rapid, targeted mass reduction (Row 1), minimal spillover (Row 2), and stable, monotonic accuracy gains (Row 3).
        \textbf{Baselines (Col 2-3)} show less efficient refinement and significant spillover, indicative of policy oscillation.
        \textbf{Ablations (Col 4-5)} confirm the necessity of core components; removing clustering or verification leads to inefficient or unstable repairs.
        \textbf{Random Edit (Col 6)} fails to make meaningful progress.
    }
    \label{fig:comparative_dynamics}
\end{figure*}

Analysis of Figure~\ref{fig:comparative_dynamics} confirms CE-Graph's superiority. Its full version (Col 1) achieves the fastest reduction in target failure mass (Row 1) while maintaining a flat, minimal spillover mass (Row 2), leading to the smoothest increase in validation accuracy (Row 3). In contrast, baselines (Col 2-3) exhibit significant spillover (policy oscillation), and ablations (Col 4-5) reveal the necessity of the clustering and verification components for efficient and stable refinement. This detailed decomposition provides strong visual evidence that CE-Graph's performance stems from its principled, stable, and targeted process of reducing failure mass.

\section{Case Study}
\label{sec:case}

\paragraph{Case 1}
In this probability and combinatorics problem, our method correctly executes all steps of a complex calculation. In contrast, both \textbf{AFlow and MaAS make calculation errors} in their respective approaches. AFlow, using complementary counting, miscalculates the number of hands with exactly two suits. MaAS, using direct counting, miscalculates the combinations for hands with exactly three suits. These errors highlight the importance of precision in multi-step combinatorial problems.

\begin{tcolorbox}[
    breakable,
    colback=black!5,
    colframe=black!75,
    title=\textbf{Case 1: Mike's Cards},
    fonttitle=\bfseries\large\sffamily,
    arc=2mm
]

\begin{correctbox}
The solution correctly applies direct counting by summing the number of hands with exactly 3 suits and exactly 4 suits.
\begin{itemize}
    \item \textbf{Hands with 3 suits:} $1,529,112$
    \item \textbf{Hands with 4 suits:} $685,464$
\end{itemize}
The total number of favorable hands is $1,529,112 + 685,464 = 2,214,576$. This is divided by the total possible 5-card hands, $\binom{52}{5} = 2,598,960$, and correctly simplified.
\begin{equation*}
P(\text{At least 3 suits}) = \frac{2,214,576}{2,598,960} = \frac{507}{595}
\end{equation*}
The final answer is \finalanswer{$\mathbf{507 / 595}$}.
\end{correctbox}

\begin{incorrectbox}
Both methods attempt valid strategies but fail due to \textbf{calculation errors}.
\begin{itemize}
    \item \textbf{AFlow (Complementary Counting):} This approach is valid, but it miscalculates the number of hands drawn from exactly two suits during an intermediate step, leading to an incorrect total of complementary cases. The final answer given is \finalanswer{$137153 / 162435$}.
    \item \textbf{MaAS (Direct Counting):} This approach is also valid, but it contains calculation errors when determining the number of hands from exactly three suits (e.g., the (3,1,1) distribution), resulting in an incorrect total. The final answer given is \finalanswer{$46173 / 54145$}.
\end{itemize}
\end{incorrectbox}
\end{tcolorbox}

\paragraph{Case 2}
For this logic and statistics problem, our method correctly interprets that the set of the ``five tallest buildings'' changes when a new, taller building is introduced, displacing the shortest one. However, both \textbf{AFlow and MaAS misinterpret the question}, assuming the new building is simply added to the original five. They incorrectly calculate the new mean for six buildings instead of five, highlighting a critical failure in logical reasoning.

\begin{tcolorbox}[
    breakable,
    colback=black!5,
    colframe=black!75,
    title=\textbf{Case 2: The five tallest buildings in Los Angeles},
    fonttitle=\bfseries\large\sffamily,
    arc=2mm
]
\begin{questionbox}
The five tallest buildings in Los Angeles in 1985 had a mean height of 733 feet. The shortest of the five was 625 feet. If a new building were constructed with a height of 885 feet, by how many feet would it increase the mean height of the \textbf{five tallest} buildings of the city?
\end{questionbox}

\begin{correctbox}
This solution correctly assumes the new building (885 ft) replaces the shortest of the original five (625 ft), maintaining a group of five buildings. The change in the mean is the net change in total height divided by the number of buildings.
\begin{align*}
\text{Change in Total Height} &= 885 \, \text{ft} - 625 \, \text{ft} = 260 \, \text{ft} \\
\text{Increase in Mean} &= \frac{\text{Change in Total Height}}{\text{Number of Buildings}} = \frac{260}{5} = 52 \, \text{ft}
\end{align*}
The final answer is \finalanswer{$\mathbf{52}$}.
\end{correctbox}

\begin{incorrectbox}
This common mistake stems from a \textbf{misinterpretation of the question}. Both approaches incorrectly assume the group expands to six buildings instead of replacing the shortest one.
\begin{align*}
\text{Original Total Height} &= 5 \times 733 = 3665 \, \text{ft} \\
\text{New Total Height (for 6 buildings)} &= 3665 + 885 = 4550 \, \text{ft} \\
\text{New Mean (for 6 buildings)} &= \frac{4550}{6} \approx 758.33 \, \text{ft} \\
\text{Incorrect Increase} &= 758.33 - 733 = 25.33 \, \text{ft}
\end{align*}
This logical error leads to the incorrect answer of \finalanswer{$\mathbf{25.33}$}.
\end{incorrectbox}
\end{tcolorbox}

\paragraph{Case 3}
This case presents a complex conditional probability problem involving multiple dice rolls and specific target outcomes. Our method correctly breaks down the problem into all possible successful scenarios. In contrast, both \textbf{AFlow and MaAS demonstrate a significant failure in logical reasoning}. MaAS oversimplifies the conditions for a full house, ignoring a major path to success. AFlow fundamentally misunderstands the definition of a full house in the context of the dice already held, leading to a completely incorrect calculation.

\begin{tcolorbox}[
    breakable,
    colback=black!5,
    colframe=black!75,
    title=\textbf{Case 6: Probability of a Full House},
    fonttitle=\bfseries\large\sffamily,
    arc=2mm
]
\begin{questionbox}
Each of five standard, six-sided dice is rolled once. Two of the dice come up the same, but the other three are all different from those two and different from each other. The pair is set aside, and the other three dice are re-rolled. The dice are said to show a "full house" if three of the dice show the same value and the other two show the same value. What is the probability that after the second set of rolls, the dice show a full house?
\end{questionbox}

\begin{correctbox}
The solution correctly identifies that with a pair already set aside, a full house can be formed in two main ways by re-rolling the three dice:
\begin{itemize}
    \item \textbf{Case A: Form a new three-of-a-kind.} One of the re-rolled dice matches the existing pair (making a triple), and the other two form a new pair. The probability for this is $5 \times \frac{3}{216} = \frac{15}{216}$.
    \item \textbf{Case B: The three re-rolled dice are all the same.} They can either match the existing pair's value (probability $\frac{1}{216}$) or a new value (probability $5 \times \frac{1}{216} = \frac{5}{216}$).
\end{itemize}
The total probability is the sum of these successful outcomes.
\begin{equation*}
P(\text{Full House}) = \frac{15}{216} + \frac{1}{216} + \frac{5}{216} = \frac{21}{216} = \frac{7}{72}
\end{equation*}
The final answer is \finalanswer{$\mathbf{7/72}$}.
\end{correctbox}

\begin{incorrectbox}
The competing methods failed due to critical misinterpretations of the problem's conditions.
\begin{itemize}
    \item \textbf{MaAS's Error:} This approach incorrectly assumes the only way to get a full house is for the three re-rolled dice to all land on the same new value, different from the original pair. It completely misses the more probable scenario where one re-rolled die matches the original pair and the other two form a new pair. This oversimplification leads to an incomplete calculation. The final answer given is \finalanswer{$\mathbf{5/216}$}.
    \item \textbf{AFlow's Error:} This method fundamentally misunderstands the goal. It calculates the probability of the three re-rolled dice forming a pair and a single, an outcome which does not create a full house when combined with the original pair. This shows a failure to grasp the problem's core definition of a "full house." The final answer given is \finalanswer{$\mathbf{5/36}$}.
\end{itemize}
\end{incorrectbox}
\end{tcolorbox}

\section{Proof of Theorem 1}
\label{sec:Theorem1}

\subsection{Restating the Theorem and Assumptions}
\textbf{Theorem 1 (Greedy Reduction Bound).} Let $\Delta_t$ be selected as in Equation (2), i.e., the greedy edit that maximizes the expected failure mass reduction. Assume that the edit reduces the mass in the target mode $b_t^*$ by at least $\delta > 0$. Then,
\[
M(W_{t+1}) \leq M(W_t) - \delta + \epsilon,
\]
where $\epsilon = O(L \cdot B \cdot \mu(\mathcal{F} \setminus b_t^*))$ bounds spillover effects to non-target regions (with $\mu$ denoting the measure of the space).

\textbf{Assumptions:}
The failure density $p(s \mid W)$ is Lipschitz continuous with constant $L > 0$: For any fixed $W$, and for all $s, s' \in \mathcal{F}$,
    \[
    |p(s \mid W) - p(s' \mid W)| \leq L \|s - s'\|.
    \]
    (This controls how much the density varies across the space $\mathcal{F}$.)
Edits $\Delta$ have bounded impact: For any $s \in \mathcal{F}$,
    \[
    \|p(s \mid W \oplus \Delta) - p(s \mid W)\|_\infty \leq B,
    \]
    where $B > 0$ is a uniform bound on the pointwise change in density due to the edit. (This assumes edits don't arbitrarily disrupt the entire density.)
The space $\mathcal{F}$ is partitioned into modes (e.g., clusters from GMM), with $b_t^*$ being the densest target mode at step $t$.
The measure $\mu$ is finite over $\mathcal{F}$ (e.g., $\mathcal{F}$ is compact or has finite volume). Alternatively, we can bound spillover via norms (e.g., $\epsilon \leq \|\Delta p\|_1$ or $\epsilon \leq \|\Delta p\|_{\mathrm{BL}} \cdot C$, where $\mathrm{BL}$ is the bounded-Lipschitz dual norm), providing equivalent probabilistic upper bounds that avoid depending on the finiteness of $\mu(\mathcal{F})$.
The edit $\Delta_t$ specifically reduces the integral mass over $b_t^*$ by at least $\delta > 0$:
    \[
    \int_{b_t^*} p(s \mid W_{t+1}) \, ds \leq \int_{b_t^*} p(s \mid W_t) \, ds - \delta.
    \]
 (Localized Propagation Assumption) There exists a constant $c > 0$ such that the edit-induced density change $\Delta p$ has support contained in $b_t^*$'s $r$-neighborhood, with $r \leq c \, B/L$. Thus,
    \[
    \epsilon = \int_{\mathcal{F} \setminus b_t^*} |\Delta p| \leq B \cdot \mu\big(\mathcal{N}_r(b_t^*) \setminus b_t^*\big) \lesssim (B^2/L) \cdot S(\partial b_t^*).
    \]
    Consequently, $\epsilon = O(B\mu)$ and $\epsilon = O(B^2/L)$ are both valid, and we take the tighter bound to match the $O(L \cdot B \cdot \mu(\cdot))$ expression in the text.

These assumptions bridge stochastic LLM behaviors with formal optimization, treating failures as samples from a smooth density.

\subsection{Proof of the Bound (Single Step Reduction)}
We start by decomposing the expected failure mass $M(W)$ over the target mode $b_t^*$ and the rest of the space.
\textbf{Decompose the Mass:} By definition,
    \[
    M(W_t) = \int_{\mathcal{F}} p(s \mid W_t) \, ds = \int_{b_t^*} p(s \mid W_t) \, ds + \int_{\mathcal{F} \setminus b_t^*} p(s \mid W_t) \, ds.
    \]
    Similarly,
    \[
    M(W_{t+1}) = \int_{b_t^*} p(s \mid W_{t+1}) \, ds + \int_{\mathcal{F} \setminus b_t^*} p(s \mid W_{t+1}) \, ds.
    \]
\textbf{Handle the Target Mode Reduction:} By the theorem's assumption (the greedy edit targets $b_t^*$ effectively),
    \[
    \int_{b_t^*} p(s \mid W_{t+1}) \, ds \leq \int_{b_t^*} p(s \mid W_t) \, ds - \delta.
    \]
    This is the "greedy" part: the edit is chosen to flatten the density mountain in $b_t^*$, reducing its mass by at least $\delta$. (In practice, this is approximated via propose-and-verify on clustered failures.)
\textbf{Bound the Spillover to Non-Target Regions:} The change in the non-target integral is
    \[
    \int_{\mathcal{F} \setminus b_t^*} p(s \mid W_{t+1}) \, ds = \int_{\mathcal{F} \setminus b_t^*} p(s \mid W_t) \, ds + \int_{\mathcal{F} \setminus b_t^*} [p(s \mid W_{t+1}) - p(s \mid W_t)] \, ds.
    \]
    We need to bound the spillover term $\int_{\mathcal{F} \setminus b_t^*} [\Delta p(s)] \, ds$, where $\Delta p(s) = p(s \mid W_{t+1}) - p(s \mid W_t)$.

    Since the edit's impact is bounded pointwise by $B$, $|\Delta p(s)| \leq B$ for all $s$. However, to incorporate $L$ reasonably, we invoke the localized propagation assumption: the primary effect of the edit is localized, with $\operatorname{supp}(\Delta p) \subset \mathcal{N}_r(b_t^*)$, where $\mathcal{N}_r$ is $b_t^*$'s $r$-neighborhood, and $r \leq c \cdot B/L$ (perturbation amplitude $B$ attenuated by Lipschitz decays over distance no more than $\sim B/L$ in a band-like region). Thus,
    \[
    \left| \int_{\mathcal{F} \setminus b_t^*} \Delta p(s) \, ds \right| \leq \int_{\mathcal{F} \setminus b_t^*} |\Delta p(s)| \, ds \leq B \cdot \mu(\mathcal{N}_r(b_t^*) \setminus b_t^*) \leq B \cdot S(\partial b_t^*) \cdot r \lesssim (B^2/L) \cdot S(\partial b_t^*),
    \]
    where $S(\partial b_t^*)$ is the boundary's "surface area/measure." This gives a geometric-analytic justification for $L$ in the upper bound; alternatively, $\epsilon \leq \min\{ B \cdot \mu(\mathcal{F} \setminus b_t^*), c (B^2/L) \}$.
 \textbf{Combine Terms:} Putting it together,
    \[
    M(W_{t+1}) \leq \left( \int_{b_t^*} p(s \mid W_t) \, ds - \delta \right) + \left( \int_{\mathcal{F} \setminus b_t^*} p(s \mid W_t) \, ds + \epsilon \right) = M(W_t) - \delta + \epsilon.
    \]
    This completes the single-step bound.

\subsection{Proof of Convergence to a Stationary Point (Under Repeated Application)}
Now, we extend to multiple iterations, showing the process converges to a local minimum where no further greedy edits reduce the mass significantly.

\textbf{Iterative Reduction:} Apply the single-step bound repeatedly: Starting from $W_0$, after $T$ steps,
    \[
    M(W_T) \leq M(W_0) - \sum_{t=0}^{T-1} (\delta_t - \epsilon_t),
    \]
    where $\delta_t > 0$ and $\epsilon_t = O(L \cdot B \cdot \mu(\mathcal{F} \setminus b_t^*))$ at each step.
\textbf{Bounded Below and Monotonic Decrease:} Since $M(W) \geq 0$ (as it's an integral of a density), the sequence $M(W_t)$ is bounded below. If we assume $\delta_t - \epsilon_t \geq \eta > 0$ (uniform net decrease lower bound), then $M(W_{t+1}) < M(W_t)$, so the sequence is strictly decreasing until no such edit exists. Otherwise, denote
    \[
    g_t := \max_{\Delta \in \mathcal{A}(W_t, \mathcal{O})} \big( M(W_t) - M(W_t \oplus \Delta) \big),
    \]
    and use $g_t \leq \tau$ as the stopping condition to guarantee finite-step convergence to a $\tau$-stationary point; meanwhile, $M(W_t)$ is monotonically bounded and converges.

\textbf{Convergence to Stationary Point:} The process stops when no mode $b_t^*$ admits an edit with $\delta_t > \epsilon_t$ (or $\delta_t > 0$ if $\epsilon_t$ is negligible). This is a stationary point: A workflow $W^*$ where the greedy objective in Eq. (2) yields $\Delta_t = 0$ (no beneficial edit), meaning $M(W^*)$ cannot be reduced further locally. Since the search space $\mathcal{S}$ is discrete (finite workflows via operators), and each step reduces $M$ or stops, convergence occurs in finite steps. (Non-convexity may lead to local minima, but the bound holds.) Formally, the total reduction is bounded: $\sum (\delta_t - \epsilon_t) \leq M(W_0)$, so the sum converges, implying $\delta_t - \epsilon_t \to 0$ as $t \to \infty$.

This proof aligns with counterexample-guided methods like CEGAR~\citep{hidvégi2024cigarcostefficientprogramrepair}, adapted for stochastic densities. If $\epsilon$ is small (e.g., localized edits), convergence is efficient.

\section{Supplementary Proofs for 
\label{sec:SupplementaryTheorem1}
Section~\ref{sec:method}}

\subsection{Lemma on Suitability of Failure Signature Mapping (Section~\ref{sec:space_construction})}
\label{app:lemma_phi}
\textbf{Lemma 1.} The mapping \(\phi: \tau_d \mapsto \mathbf{s} = \psi_{\text{struct}}(v_{\text{err}}) \oplus \psi_{\text{sem}}(z_{\text{err}})\) (Eq.~\ref{eq:failure_signature}) is injective under the assumption that \(\psi_{\text{struct}}\) is a one-hot encoding over distinct nodes \(V\) and \(\psi_{\text{sem}}\) is a continuous embedding (e.g., BERT-like) with distinct outputs for unique error messages \(\mathcal{Z}\). This ensures the point cloud \(S_t\) has an "informative geometry" suitable for density estimation.

\textbf{Proof.} Injectivity holds if \(\phi(\tau_d) = \phi(\tau_{d'})\) implies \(\tau_d = \tau_{d'}\). Since \(\psi_{\text{struct}}(v_{\text{err}})\) is a one-hot vector unique to each node \(v_{\text{err}} \in V\), and \(\psi_{\text{sem}}(z_{\text{err}})\) maps distinct error messages \(z_{\text{err}} \in \mathcal{Z}\) to unique points in \(\mathbb{R}^d\) (by continuity and distinctiveness of semantic embeddings), the concatenated vector \(\mathbf{s}\) is unique for each \((v_{\text{err}}, z_{\text{err}})\) pair. The geometry is informative because \(\psi_{\text{struct}}\) provides orthogonal subspaces (node-specific patterns), and \(\psi_{\text{sem}}\) clusters semantically similar errors, enabling GMM clustering as in Section~\ref{sec:solving_objective}. \(\square\)

\subsection{Theorem on GMM Consistency and Mode Error (Sections~\ref{sec:solving_objective})}
\label{app:theorem_gmm}
\textbf{Theorem 2.} Under the assumption that failure signatures \(\{\mathbf{s}_i\}_{i=1}^{|S_t|}\) are i.i.d. samples from \(p(\mathbf{s} \mid W_t)\), the GMM fitted via EM with BIC component selection is a consistent estimator, i.e., \(\mathbb{E}[\|\hat{p}_\theta - p\|_1] \to 0\) as \(|S_t| \to \infty\). The densest mode \(\hat{b}_t^* = \arg\max_{b_k} \hat{\pi}_k\) (Eq.~\ref{eq:gradient_estimation_combined}) approximates the true max-density mode with expected error \(O(1/\sqrt{|S_t|})\).

\textbf{Proof.} Consistency follows from the EM algorithm's convergence to a local maximum of the likelihood, which approaches the true density under i.i.d. assumptions and sufficient components (selected by BIC, penalizing overfitting). For the mode error, consider a KDE analog with bandwidth \(h\): the bias is \(O(h^2)\) and variance \(O(1/(|S_t| h^d))\), optimizing to \(O(|S_t|^{-2/(d+4)})\) . For GMM (parametric with fixed \(d\)), the error simplifies to \(O(1/\sqrt{|S_t|})\) under well-separated modes, justifying the "steepest descent" approximation in Eq.~\ref{eq:greedy_mass_reduction}. \(\square\)

\subsection{Proposition on Monte Carlo Verification (Section~\ref{sec:solving_objective})}
\label{app:prop_monte_carlo}
\textbf{Proposition 3.} The empirical success rate \(V(\Delta_i) = \frac{1}{K} \sum_{k=1}^K \mathbb{I}\left[ \text{Verify}(\text{Execute}(W_t \oplus \Delta_i, x_k), y_k) = 1 \right]\) (Eq.~\ref{eq:verification_utility}) is an unbiased estimator of the true success probability, with variance \(O(1/K)\), aligning with Theorem~\ref{thm:greedy_bound}'s spillover minimization.

\textbf{Proof.} Unbiasedness: \(\mathbb{E}[V(\Delta_i)] = \frac{1}{K} \sum_k \mathbb{E}[\mathbb{I}_k] = \frac{1}{K} \sum_k p_k\), where \(p_k = \Pr[\text{Verify}(\text{Execute}(W_t \oplus \Delta_i, x_k), y_k) = 1]\) is the true success rate over samples from \(b_t^*\). Variance: For i.i.d. \(\mathbb{I}_k \sim \text{Bernoulli}(p_k)\), \(\Var[\mathbb{I}_k] \leq 1/4\) (maximum for Bernoulli), so \(\Var[V(\Delta_i)] = \frac{1}{K^2} \sum_k \Var[\mathbb{I}_k] \leq \frac{1}{4K} = O(1/K)\). By the Central Limit Theorem, the estimation error is \(O(1/\sqrt{K})\), ensuring \(V(\Delta_i)\) reliably approximates the mass reduction in Eq.~\ref{eq:greedy_mass_reduction}. \(\square\)

\section*{Statement on the Use of AI Assistance}
In the preparation of this manuscript, we employed a Large Language Model (LLM) as a research and writing assistant. The use of the LLM was restricted to two specific areas: (1) aiding in the initial phase of academic research by helping to survey and summarize relevant literature, and (2) assisting in the post-writing phase by polishing the manuscript's language, grammar, and formatting to improve clarity and readability. 
\end{document}